\documentclass{article} 
\usepackage{iclr2027_conference,times}

\usepackage{amsmath,amsfonts,bm}

\def\eqref#1{equation~\ref{#1}}

\def\1{\bm{1}}

\DeclareMathAlphabet{\mathsfit}{\encodingdefault}{\sfdefault}{m}{sl}
\SetMathAlphabet{\mathsfit}{bold}{\encodingdefault}{\sfdefault}{bx}{n}

\usepackage{hyperref}
\usepackage{url}

\usepackage{amsmath}
\usepackage{amssymb}
\usepackage{amsthm}
\usepackage{bbm}
\usepackage{algorithm}
\usepackage{algpseudocode}
\usepackage{xcolor}
\usepackage{mathtools}
\usepackage{booktabs}
\usepackage{multirow}
\usepackage{makecell}
\usepackage{graphicx}
\usepackage{caption}
\usepackage{subcaption}
\usepackage{wrapfig}
\usepackage{thmtools}

\newcommand{\always}{\mathsf{G}\,}
\newcommand{\event}{\mathsf{F}\,}
\newcommand{\until}{\;\mathsf{U}\;}

\title{
SpecMAGIC: Specification-Guided Multi-Task Generalization in Multi-Agent Reinforcement Learning
}

\iclrfinalcopy

\author{
    Zijian Guo$^1$, 
    \.Ilker I\c{s}{\i}k$^2$,
    H. M. Sabbir Ahmad$^1$, Wenchao Li$^{1, 2}$ \\
    $^1$Division of Systems Engineering, Boston University \\
    $^2$Department of Electrical and Computer Engineering, Boston University \\
    \texttt{\{zjguo}, \texttt{iilker}, \texttt{sabbir92}, \texttt{wenchao\}@bu.edu}
}

\begin{document}

\maketitle
\pagestyle{fancy}
\fancyhf{}
\fancyhead[L]{Preprint}
\fancyfoot[C]{\thepage}
\renewcommand{\headrulewidth}{0.4pt}

\begin{abstract}
    Multi-task generalization is an important yet underexplored problem in multi-agent reinforcement learning (MARL). 
    Many real-world applications require multi-agent systems to generalize across complex, temporally extended tasks with diverse inter-agent coordination requirements. 
    We introduce a specification language that combines temporal-logic task specifications with operators for expressing such requirements. 
    We propose SpecMAGIC, a hierarchical framework that leverages the automaton structure of task specifications for subtask decomposition. 
    A centralized high-level planner identifies and allocates compatible subtask sequences across agents, while decentralized low-level policies execute individual subtasks. 
    This enables zero-shot generalization to new tasks by composing reusable subtask policies according to new specifications, without retraining. 
    We establish soundness and probabilistic completeness of the high-level planner. Experiments on multi-agent navigation tasks show that SpecMAGIC substantially outperforms existing approaches on unseen task specifications and maintains strong performance as the number of agents increases.
    The code is available at \href{https://github.com/BU-DEPEND-Lab/SpecMAGIC}{https://github.com/BU-DEPEND-Lab/SpecMAGIC}.
\end{abstract}

\section{Introduction}
\label{sec:intro}

Multi-task multi-agent reinforcement learning (MTMARL) aims to learn multi-agent policies that generalize across diverse tasks~\citep{omidshafiei2017deep, chai2025survey}.
Recent methods have explored this setting in applications including multi-robot systems~\citep{wang2023multi, li2024multi, zhu2025multi} and autonomous driving~\citep{berahman2025multi, wu2026risk}. 
However, zero-shot generalization to complex, temporally extended tasks with diverse inter-agent coordination requirements, as arise in many real-world applications, remains largely underexplored. 
To capture this setting, we consider multi-agent systems where each agent is assigned individual tasks, while the team must additionally satisfy coordination requirements such as dependency and count constraints.
Moreover, at test time, both the individual tasks and the coordination structure among agents can differ from those encountered during training.
For example, warehouse robots that have learned to deliver packages need to follow different user-specified delivery sequences, where some deliveries must precede others and certain deliveries require a specified number of robots.

Temporal logic provides a principled framework for specifying multi-agent behaviors over time. 
In particular, linear temporal logic (LTL) has been widely used to express complex temporally extended tasks and to guide the learning of policies that satisfy them~\citep{leon2020extended, liu2025guiding, yalcinkaya2025automata}. 
Recent approaches exploit the structure of automata corresponding to LTL specifications~\citep{vardi1986automata} to decompose them into subtasks or subtask sequences~\citep{vaezipoor2021ltl2action, qiu2023instructing, jackermeier2025deepltl}, enabling generalization to unseen specifications by composing previously learned subtask policies according to the structure of a new specification.
This paradigm can, in principle, be extended to multi-agent specifications by combining individual tasks and inter-agent coordination requirements into a global specification and applying the same decomposition procedure. 
However, the automaton corresponding to a global specification can be exponential in the size of the specification, and composing agent-specific requirements can further induce a combinatorial blow-up in the joint state space. 
Consequently, explicit automaton construction and planning become increasingly intractable as the number of agents and specification complexity grow.

In this work, we propose SpecMAGIC, a hierarchical framework for zero-shot \textbf{Spec}ification-guided \textbf{M}ulti-\textbf{A}gent task \textbf{G}eneralization with \textbf{I}nter-agent \textbf{C}oordination requirements.
We introduce specification operators for dependency and count constraints that can be translated into standard LTL formulas. allowing individual tasks and coordination requirements to be represented within a unified specification framework.
To avoid constructing a global automaton over all agents and coordination requirements, SpecMAGIC decomposes the specifications into \textit{reach-avoid} subtask sequences and plans directly over the resulting subtask space. 
A centralized sampling-based high-level planner then identifies, allocates and sequences compatible
subtasks across agents.
The resulting plan is executed by decentralized subtask-conditioned RL policies, 
compatible with a variety of goal-conditioned RL policies with subtask embeddings~\citep{vaezipoor2021ltl2action, qiu2023instructing, jackermeier2025deepltl, guo2026one}.
This hierarchical approach enables learned subtask policies to be recombined at test time to satisfy unseen individual tasks and coordination constraints, including both finite- and infinite-horizon specifications. 
Our main contributions are summarized as follows:
\vspace{-1.5mm}

\begin{itemize}
    \item We study zero-shot generalization to complex multi-agent LTL specifications that combine individual tasks with inter-agent coordination requirements, and introduce new specification operators to conveniently express these coordination constraints.
    \item We develop a scalable task-allocation approach that avoids the exponential complexity associated with explicit global product-automaton construction, and establish soundness and probabilistic completeness of the proposed planner.
    \item We conduct comprehensive experiments on diverse multi-agent task specifications with varying numbers of agents in a navigation environment, demonstrating consistent improvements over state-of-the-art baselines in zero-shot task generalization and scalability.
\end{itemize}

\vspace{-1mm}
\section{Related Work}
\label{sec:related-work}
\vspace{-1mm}

\textbf{Specification-driven RL.}
There has been growing interest in learning RL policies that satisfy tasks expressed through formal specifications.
Early work primarily focuses on a single fixed specification~\citep{hasanbeig2018LogicallyConstrained, shao2023Sample, shah2024LTLConstrained}.
More recently, zero-shot generalization to unseen specifications has been studied through direct specification or automaton encoding~\citep{vaezipoor2021ltl2action, yalcinkaya2024compositional} and through subtask decomposition~\citep{qiu2023instructing, liu2024Skill, jackermeier2025deepltl, guo2026one}.
In multi-agent settings, formal specifications have been used to coordinate multiple agents toward a global/system specification through hand-designed~\citep{neary2020reward}, heuristic~\citep{smith2023automatic}, or learned~\citep{shah2025learning} decompositions.
Recent work extends formal specifications to MTMARL over a \textit{given} collection of tasks~\citep{hammond2021multi, liu2025guiding}.
We instead focus on zero-shot generalization to \textit{unseen} specifications; the closest work, ACC-MARL~\citep{yalcinkaya2025automata}, also considers unseen specifications but is limited to finite-horizon single-agent tasks, whereas we consider both finite- and infinite-horizon specifications with explicit inter-agent coordination requirements.

\textbf{Multi-agent task decomposition and allocation.} 
In multi-agent systems with temporal-logic specifications, task decomposition and allocation are widely used to synthesize coordinated plans, 
ranging from independently assigned local specifications to global tasks with temporal dependencies, collaborative execution, and capability constraints~\citep{guo2015multi, liu2024fast, li2025task, luo2025simultaneous}.
These methods primarily address planning for a given specification and typically assume known system dynamics.
To improve planning efficiency, 
sampling-based approaches incrementally explore the product-automaton state space rather than exhaustively searching  it~\citep{kantaros2018sampling, kantaros2020stylus}.
Hierarchical methods have also been developed for task decomposition and allocation in MTMARL~\citep{wang2021rode, iqbal2022alma, li2024hierarchical}, but do not consider temporal-logic specifications.
In contrast, our high-level planner samples directly over reach-avoid subtask sequences derived from the component specifications,
avoiding construction of the global product automaton, while decentralized policies learn to execute reusable subtasks under unknown dynamics.

\vspace{-1mm}
\section{Preliminaries}
\label{sec:preliminaries}
\vspace{-1mm}

\textbf{Markov Games.} In this work, we consider partially observable Markov games (POMGs)~\citep{littman1994markov}. A POMG with $N$ agents is defined as the tuple $\mathcal{M} := (\mathcal{S}, \mathcal{A}, P, \gamma, d_0)$, where 
$\mathcal{S} = \mathcal{S}_1 \times \cdots \times \mathcal{S}_N$ denotes the joint state space, 
$\mathcal{A} = \mathcal{A}_1 \times \cdots \times \mathcal{A}_N$ denotes the joint action space, 
$\mathcal{O} = \mathcal{O}_1 \times \cdots \times \mathcal{O}_N$ denotes the joint observation space.
The transition dynamics are given by $P:\mathcal{S}\times \mathcal{A} \times \mathcal{S} \mapsto [0,1]$, 
$\gamma \in [0,1)$ is the discount factor, 
and $d_0 \in \Delta(\mathcal{S})$ is the initial state distribution. 
Each agent $i$ receives its local observation correlated with the state $o_i: \mathcal{S}\mapsto \mathcal{O}_i$ and selects an action according to a policy $\pi_i:  \mathcal{O}_i \times \mathcal{A}_i \mapsto [0, 1]$. The joint policy is denoted by $\pi := (\pi_1, \pi_2, \dots, \pi_N)$. 
Unlike standard formulations, we do not explicitly define a reward function as part of the POMG since the reward is induced by the task specification provided at test time.

\textbf{Linear Temporal Logic.}
LTL~\citep{pnueli1977temporal} provides a formalism for expressing temporal properties over system behaviors.
Given a finite set of atomic propositions $AP$, an LTL formula is recursively constructed using Boolean operators negation $\neg$, conjunction $\land$, and disjunction $\lor$, together with the temporal operators ``until'' ($\mathsf{U}$), ``eventually'' ($\mathsf{F}$), and ``always'' ($\mathsf{G}$):
\begin{equation*}
    \varphi := a \mid \neg \varphi \mid \phi \land \varphi \mid \phi \lor \varphi \mid \event \varphi \mid \always \varphi \mid \phi \until \varphi
\end{equation*}
Intuitively, the formula $\varphi_1 \until \varphi_2$ holds if $\varphi_1$ is satisfied at all time steps prior to the first occurrence of $\varphi_2$. 
The operator $\event \varphi$ holds if $\varphi$ is satisfied at some future time, while $\always \varphi$ holds if $\varphi$ is satisfied from the current step onward. 
For each agent $i$, let $AP_i$ denote its set of atomic propositions, and let $AP=\bigcup_{i=1}^N AP_i$ denote the overall proposition set.
Each agent has access to a known labeling function $L_i:\mathcal O_i\mapsto 2^{AP_i}$ that maps its local observation to the set of atomic propositions that hold at the current step, i.e., the true propositions can be inferred from observations.
We write $L(o)=\bigcup_{i=1}^N L_i(o_i)$ for the propositions satisfied given the joint observation $o=(o_1,\ldots,o_N)$.
For a trajectory $\tau=s_0s_1\ldots$ with corresponding joint observations $o_0o_1\ldots$, its trace is given by $\operatorname{Tr}(\tau)=L(o_0)L(o_1)\ldots$, and $\tau\models\varphi$ iff this trace satisfies $\varphi$.
Accordingly, the probability of satisfying $\varphi$ under policy $\pi$ is
$\Pr(\pi\models\varphi)=\mathbb E_{\tau\sim\pi}\!\left[\mathbbm{1}[\tau\models\varphi]\right]$.
We are interested in learning policies that maximize this satisfaction probability for arbitrary $\varphi$.

\textbf{Büchi automata and subtasks.} 
We represent an LTL specification $\varphi$ using a (non-deterministic) Büchi automaton~\citep{buchi1966symposium} $\mathcal B_\varphi=(\mathcal Q,\Sigma,\delta,\mathcal F,q_0)$, where $\mathcal Q$ is a finite set of automaton states, $\Sigma=2^{AP}$ is the alphabet, $\delta\subseteq\mathcal Q\times\Sigma\times\mathcal Q$ is the transition relation, $\mathcal F\subseteq\mathcal Q$ is the set of accepting states, and $q_0$ is the initial state.
For an infinite word $\omega=\alpha_0\alpha_1\ldots\in\Sigma^\omega$, a run $r=q_0q_1\ldots$ satisfies $(q_i,\alpha_i,q_{i+1})\in\delta$ for every $i\geq0$ and is accepting iff it visits an accepting state infinitely often.
To keep track of progress toward specification satisfaction, \textit{reach-avoid subtasks}~\citep{jackermeier2025deepltl}, denoted by $\sigma=(\alpha^+,\alpha^-)$, are proposed to represent transitions over the automaton.
The reach part $\alpha^+\in\Sigma$ corresponds to a transition $(q,\alpha^+,q')\in\delta$, where $q$ and $q'$ are consecutive states along an accepting run and $q'\neq q$ when $q\notin\mathcal F$.
The avoid part $\alpha^-\in\Sigma$ corresponds to a transition $(q,\alpha^-,q')\in\delta$, where $q$ and $q'$ are not consecutive states along any accepting run.
Intuitively, the reach part specifies progress toward satisfying the specification, whereas the avoid part captures transitions that can lead to a violation.

\section{Method}
\label{sec:method}

In this section, we present our hierarchical framework, with an overview shown in Figure~\ref{fig:overview}.
We first state the generalization problem considered in this work (Section~\ref{sec:problem-statement}) and then introduce the different classes of task specifications (Section~\ref{sec:specifications}).
Next, we present the centralized sampling-based high-level planner for task allocation over reach-avoid subtask sequences and establish its soundness and probabilistic completeness (Section~\ref{sec:task-decomposition-allocation}).
Finally, we present the decentralized low-level subtask-conditioned policies (Section~\ref{sec:policy-learning}).

\subsection{Problem Statement}
\label{sec:problem-statement}
We consider multi-agent tasks in which each agent is assigned an \textit{individual-agent} specification, together with additional \textit{inter-agent dependency} and \textit{count} requirements.
A global specification combines all these specification components $\varphi_{\texttt{global}}:=\bigwedge_{i\in\mathcal I}\varphi_i$, where $I:=\{1,\ldots,N\}$ is the agent index set, $\mathcal I:=I\cup\{\texttt{dep},\texttt{cnt}\}$, $\varphi_i\sim\Phi_{\texttt{ind}}$ for $i\in I$, $\varphi_{\texttt{dep}}\sim\Phi_{\texttt{dep}}$, and $\varphi_{\texttt{cnt}}\sim\Phi_{\texttt{cnt}}$, 
where $\Phi_{\texttt{ind}}$, $\Phi_{\texttt{dep}}$, and $\Phi_{\texttt{cnt}}$ denote arbitrary distributions over these three types of specifications.
Let $\Phi_{\texttt{global}}$ denote the distribution over such global specifications.
Our goal is to learn a multi-agent policy $\pi$ that maximizes the probability of satisfying specifications drawn from $\Phi_{\texttt{global}}$
\begin{equation*}
\pi^\star
= \underset{\pi}{\arg\max} \
\mathbb{E}_{\varphi_{\texttt{global}} \sim \Phi_{\texttt{global}}, \tau \sim \pi}
\left[\mathbbm{1}[\tau \models \varphi_{\texttt{global}}]\right].
\end{equation*}

\begin{figure}[ht]
    \centering
    \includegraphics[width=0.95\linewidth,trim={130 1450 1930 1250},clip]{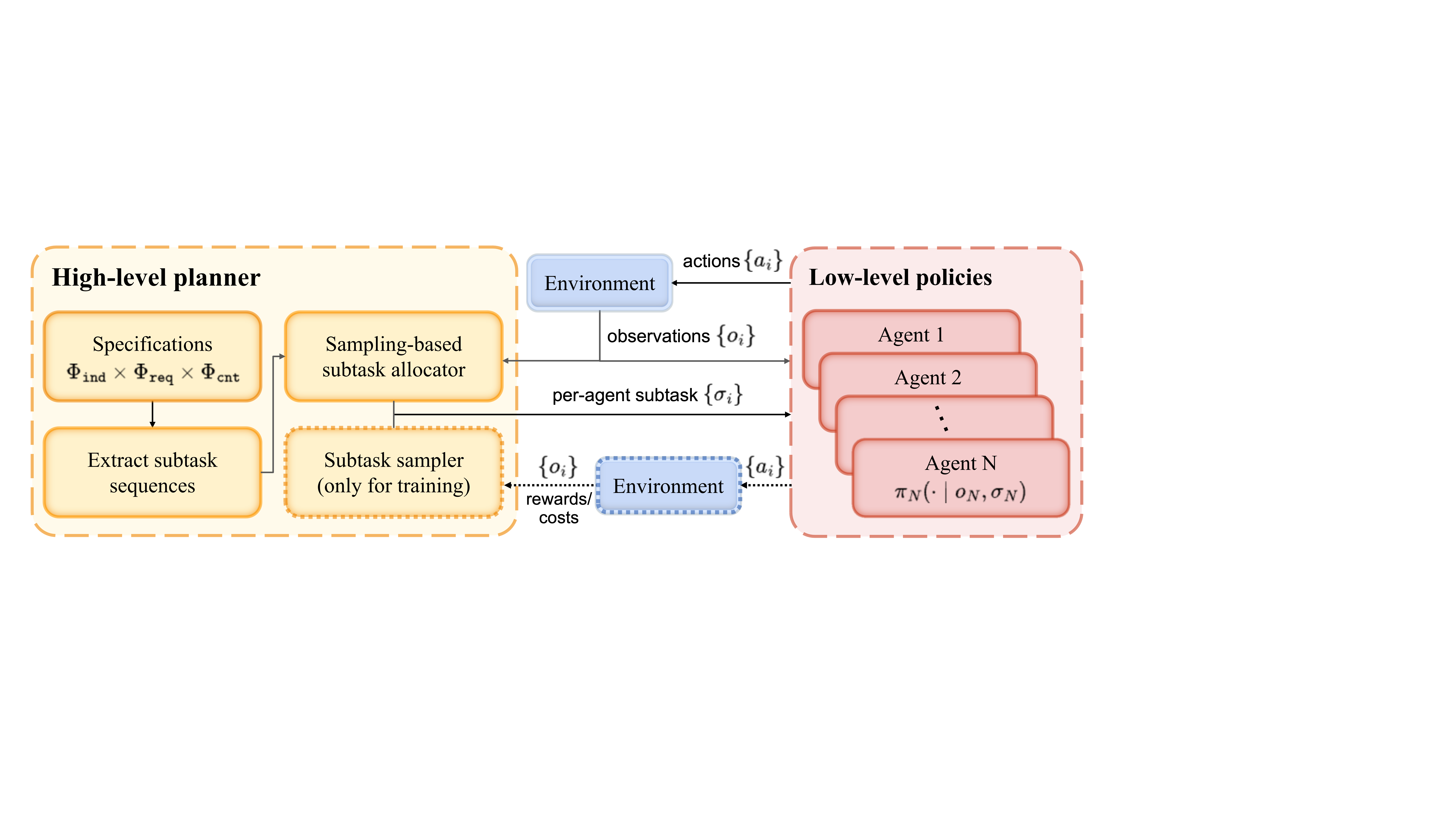}
    
    \vspace{-2mm}
    \caption{
    Overview of the proposed hierarchical framework. 
    Specifications are decomposed into reach-avoid subtasks, from which a centralized sampling-based planner constructs a feasible subtask sequence for execution by decentralized low-level policies.
    Solid and dotted blocks and arrows indicate components used at test time and during training, respectively.
    }
    \label{fig:overview}
\end{figure}

\subsection{Specification Language}
\label{sec:specifications}

We now introduce the specification language for the components of $\varphi_{\texttt{global}}$.
We use standard LTL to express individual-agent specifications, allowing each $\varphi_i$, $i\in I$, to be an arbitrary LTL formula over the agent-specific atomic propositions $AP_i$.
To express inter-agent coordination, we augment these temporal-logic task specifications with operators for dependency and count requirements.


\textbf{Dependency requirements.}
To capture dependency relations among agents, we define the operators $\texttt{Req}$ and $\texttt{EachReq}$, and specify dependency requirements by the following syntax:
\begin{equation*}
    \varphi_{\texttt{dep}}::=(\alpha \ \texttt{Req} \ \beta) \mid (\alpha \ \texttt{EachReq} \ \beta) \mid \phi_{\texttt{dep}} \land \varphi_{\texttt{dep}} \mid \phi_{\texttt{dep}} \lor \varphi_{\texttt{dep}}
\end{equation*}
where $\alpha$ and $\beta$ are Boolean formulas over atomic propositions.
We say $\tau \models \alpha \ \texttt{Req} \ \beta$ iff there exists a time $t \geq 0$ such that $\tau, t \models \beta$ and $\forall t' \in [0, t)$, $\tau, t' \not\models \alpha$ and 
$\tau \models \alpha \ \texttt{EachReq} \ \beta$ iff $\tau \models \alpha \ \texttt{Req} \ \beta$ and, for every $t \geq 0$ such that $\tau,t \models \alpha$ and either $t=0$ or $\tau,t-1 \not\models \alpha$, there exists $t' \in [0,t]$ such that $\tau,t' \models \beta$ and $\tau,t'' \not\models \alpha$ for all $t'' \in [t',t)$.
Intuitively, $\texttt{Req}$ requires the prerequisite $\beta$ before the first occurrence of $\alpha$, whereas $\texttt{EachReq}$ requires $\beta$ before every non-consecutive occurrence of $\alpha$.
In contrast, the negation of a dependency relation, e.g., $\neg (\alpha \  \texttt{Req} \ \beta)$ or $\neg (\alpha \  \texttt{EachReq} \ \beta)$, does not correspond to a meaningful dependency requirement and is therefore excluded from the syntax.

\textbf{Count requirements.}
We also introduce operators $\texttt{Sync}$ and $\texttt{Async}$ to express count requirements among agents, such as how many agents must satisfy a given condition.
The corresponding specifications are defined by the following syntax:
\begin{equation*}
    \varphi_{\texttt{cnt}} ::= 
    \texttt{Sync}[I, \alpha, m] \mid 
    \texttt{Async}[I, \alpha, m] \mid \neg \varphi_{\texttt{cnt}} \mid \phi_{\texttt{cnt}} \land \varphi_{\texttt{cnt}} \mid \phi_{\texttt{cnt}} \lor \varphi_{\texttt{cnt}}
\end{equation*}
where $\alpha$ is a Boolean formulas over atomic propositions, $I$ is the set of agents $I \subseteq [N]$, and $1 \leq m \leq |I|$.
Satisfaction of $\varphi$ by $\tau$ is defined as 
$\tau \models \texttt{Sync}[I, \alpha, m] \Leftrightarrow \exists t: | \{i \in I \mid L_i(s_{i, t}) \models \alpha\} | \geq m $ and 
$\tau \models \texttt{Async}[I, \alpha, m] \Leftrightarrow | \{i \in I \mid \exists t: L_i(s_{i, t}) \models \alpha\} | \geq m $.
Intuitively, $\texttt{Sync}$ requires at least $m$ agents to satisfy $\alpha$ at the same (discrete) timestep, and $\texttt{Async}$ only requires at least $m$ agents to satisfy $\alpha$, possibly at different timesteps.

\textbf{Connections to real-world applications.}
These operators provide a convenient way to express coordination requirements among agents that naturally arise in a wide range of multi-agent applications.
For example, in surveillance and security, a guard robot may need to inspect an area before another robot enters, which can be expressed using $\texttt{Req}$ or $\texttt{EachReq}$, while $\texttt{Async}$ can specify that a certain number of areas are visited for monitoring coverage.
In warehouse automation, $\texttt{Req}$ or $\texttt{EachReq}$ can specify that inventory availability must be confirmed before items are retrieved, while $\neg\texttt{Async}$ can prevent a certain number of robots from retrieving out-of-stock items.
In multi-robot manufacturing, $\texttt{Req}$ or $\texttt{EachReq}$ can require an object to be stabilized before manipulation, while $\neg\texttt{Sync}$ can limit the number of robots operating on the same object simultaneously.
These operators also serve as a convenient specification interface and can be translated into standard LTL formulas, allowing standard automata-theoretic constructions and operations for LTL to be directly applied; the corresponding translations are provided in Appendix~\ref{app:specs}.

\subsection{Task Decomposition and Allocation}
\label{sec:task-decomposition-allocation}

To obtain a feasible multi-agent task plan, a straightforward approach is to construct the automaton of the global specification $\varphi_{\texttt{global}}$, or equivalently the product automaton of its component specifications, and search for feasible subtask sequences.
However, the resulting state space grows exponentially with the number of agents and specification complexity~\citep{kantaros2018sampling, kantaros2020stylus, abate2026semantically}, making explicit construction and search intractable at scale.
Previous works~\citep{qiu2023instructing, jackermeier2025deepltl, guo2026one} have shown that leveraging equivalent subtask sequences derived from specifications can improve learning efficiency and generalization.
Motivated by this, we develop a sampling-based task-allocation method that operates on the reach-avoid subtask sequences of the component specifications, identifies compatible combinations of subtasks across components while respecting their coordination requirements, and incrementally constructs a feasible multi-agent plan without explicitly constructing the product automaton.
Under unknown dynamics and partial observability, we focus on finding feasible solutions and establish the soundness and probabilistic completeness of the planner.

\begin{figure*}[ht]
\centering
\vspace{-4mm}
\begin{minipage}[t]{0.56\textwidth}
\begin{algorithm}[H]
    \caption{Centralized sampling-based planner}
    \label{algo:subtask-sampling}
    \begin{algorithmic}[1]
        \Require Subtask-sequence sets $\{\mathcal{G}_i\}_{i\in\mathcal{I}}$ for specifications $\{\varphi_i\}_{i\in\mathcal{I}}$, maximum sampling iterations $K$
        \State Initialize tree $\mathcal{T}$ with the nominal (empty) root $x_{\mathrm{root}}$
        \For{$k=0,1,\ldots,K$}
            \State Sample $x_{\mathrm{parent}} \sim \mathrm{Unif}(\mathcal{T})$ \label{line:subtask-sampling-sample}
            \State $x_{\mathrm{new}} \gets \Call{Extend}{x_{\mathrm{parent}},\{\mathcal{G}_i\}_{i\in\mathcal{I}}}$ \label{line:subtask-sampling-extend}
            \State \textbf{if} $x_{\mathrm{new}}$ is \textsc{Invalid}, \textbf{then continue}
            \State \textbf{if} $x_{\mathrm{new}}\notin\mathcal{T}$, \textbf{then} add $x_{\mathrm{new}}$ with parent $x_{\mathrm{parent}}$
            \State \textbf{if} $x_{\mathrm{new}}$ is terminal, \textbf{then} \textbf{return} the root-to-$x_{\mathrm{new}}$ subtask sequence
        \EndFor
        \State \Return \textsc{Failure}
    \end{algorithmic}
\end{algorithm}
\end{minipage}
\hfill
\begin{minipage}[t]{0.42\textwidth}
    \begin{algorithm}[H]
    \caption{Extend}
    \label{algo:subtask-extend}
        \begin{algorithmic}[1]
            \Require Node $x_{\mathrm{parent}}$ and $\{\mathcal G_i\}_{i\in\mathcal I}$
            \State Identify expandable components $\mathcal I_{\mathrm{exp}}$
            \State Sample a nonempty subset $\mathcal I_{\mathrm{s}}\subseteq\mathcal I_{\mathrm{exp}}$
            \State $x' \gets x_{\mathrm{parent}}$, $A^+ \gets \{\}$, $A^- \gets \{\}$
            \For{$i\in\mathcal I$}
                \State Determine the current subtask for component $i$, update $A^+$ and $A^-$ with its reach and avoid conditions, and update $x'$ accordingly \label{algo-line:extend-loop}
            \EndFor
            \State \textbf{if} $x'$ is valid, \textbf{then return} $x'$
            \State \Return \textsc{Invalid}
        \end{algorithmic}
    \end{algorithm}
\end{minipage}
\vspace{-0.5mm}
\end{figure*}

\textbf{Sampling-based task allocation.} 
The overall process of the proposed method is summarized in Algorithm~\ref{algo:subtask-sampling}.
For each $\varphi_i$, $i \in \mathcal{I}$, let $\mathcal{G}_i=\{g_i^1,g_i^2,\ldots,g_i^{M_i}\}$ denote the set of candidate subtask sequences, where $M_i=|\mathcal G_i|$. 
Each $g_i^j\in\mathcal G_i$ is a finite sequence $g_i^j=[\sigma_{i,0}^j,\sigma_{i,1}^j,\ldots]$, where $\sigma_{i,k}^j=(A_{i,k}^{+,j},A_{i,k}^{-,j})$ denotes a reach-avoid subtask.
The planner constructs a search tree $\mathcal T$ over the subtask space, where each node $x$ represents the progress of all specification components as a set of index tuples $(i,j,k)$, with exactly one tuple for each $i\in\mathcal I$.
For each tuple, $i$ identifies the specification component, $j$ identifies the selected sequence $g_i^j$, and $k$ identifies the current subtask $\sigma_{i,k}^j$ in that sequence, with $k=|g_i^j|$ indicating that the planner has reached the end of the sequence.
We use $(i,0,0)$ when no sequence has yet been selected for component $i$, such that the root node is $x_{\mathrm{root}}=\{(i,0,0)\mid i\in\mathcal I\}$.
At each iteration, a node $x_{\mathrm{parent}}$ is uniformly sampled from $\mathcal T$ and extended using Algorithm~\ref{algo:subtask-extend} to obtain a successor $x_{\mathrm{new}}$ with compatible next subtasks.
If such a successor is found, it is added to $\mathcal T$.
A node is terminal when, for every tuple $(i,j,k)$, a sequence has been selected and the planner has reached its end, i.e., $j\neq0$ and $k=|g_i^j|$ for all $i\in\mathcal I$.
The path from $x_{\mathrm{root}}$ to a terminal node then yields a subtask sequence satisfying all specifications.
More generally, our method also applies when the specifications are provided as a collection of Büchi automata, as the planner only requires the corresponding reach-avoid subtask sequences.

\textbf{Subtask combination compatibility.}
Algorithm~\ref{algo:subtask-extend} identifies a compatible combination of next subtasks for a sampled node.
Given $x_{\mathrm{parent}}$, we define the expandable components as $\mathcal I_{\mathrm{exp}}:=$ $\{i\in\mathcal I\mid j=0$ or $k<|g_i^j|$ for the tuple $(i,j,k)\in x_{\mathrm{parent}}\}$, i.e., components for which no sequence has yet been selected, or the planner has not yet reached the end of the selected sequence.
We then sample a nonempty subset $\mathcal I_{\mathrm{s}}\subseteq\mathcal I_{\mathrm{exp}}$ of components to consider at the current planning step, since enumerating all possible subsets can be intractable.
For each $i\in\mathcal I_{\mathrm{s}}$, if its tuple is $(i,0,0)$, we sample a sequence $g_i^j\in\mathcal G_i$ and select its first subtask; otherwise, for tuple $(i,j,k)$, we select the next subtask in $g_i^j$.
The corresponding tuples are updated to record this planning progress, while the tuples for $i\notin\mathcal I_{\mathrm{s}}$ remain unchanged.
To check compatibility, for each $i\in\mathcal I_{\mathrm{s}}$, one conjunctive clause $\tilde A_i^+$ is sampled from the reach condition of its selected subtask, while $\tilde A_i^+=\emptyset$ for $i\in\mathcal I\setminus\mathcal I_{\mathrm{s}}$.
The joint reach condition is then $A_{\mathrm{joint}}^+:=\bigcup_{i\in\mathcal I}\tilde A_i^+$.
In contrast, every component contributes its full avoid condition $A_i^-$, yielding $A_{\mathrm{joint}}^-:=\bigcup_{i\in\mathcal I}A_i^-$, since omitting an unselected component's avoid condition could allow it to be satisfied during execution, after which the corresponding specification can no longer be satisfied.
The sampled successor is invalid if there exists a conjunctive clause $B\in A_{\mathrm{joint}}^-$ such that $B\subseteq A_{\mathrm{joint}}^+$.
More details of the planner can be found in Appendix~\ref{app:implementation}.

\textbf{Soundness and (probabilistic) completeness.}
Many automata-based search methods~\citep{guo2013revising, guo2015multi, qiu2023instructing, shah2024ltl, jackermeier2025deepltl} have been developed to find accepting paths that satisfy a given LTL specification.
These methods search for a finite prefix reaching an accepting state together with a cyclic suffix, and the resulting prefix-suffix paths, i.e., subtask sequences, are guaranteed to satisfy the given specification.
We leverage these methods to obtain candidate subtask sequences for each specification component over which our planner operates.
Given this property, we have the following soundness guarantee:

\begin{restatable}[Soundness]{theorem}{soundness}
\label{thm:soundness}
Let $\tau$ be a path returned by the proposed sampling-based planner.
Then $\tau$ is an accepting path satisfying the global specification
$\varphi_{\texttt{global}}$, i.e.,
$\tau \models \varphi_{\texttt{global}}$.
\end{restatable}

\vspace{-2mm}
Soundness guarantees that any path returned by the proposed method satisfies the global specification.
We next show probabilistic completeness, i.e., if a feasible path exists, the probability of finding one converges to one given sufficient iterations.
The proofs are provided in Appendix~\ref{app:theorems}.

\begin{restatable}[Probabilistic completeness]{theorem}{probcomplete}
\label{thm:prob-complete}
Suppose that there exists a feasible subtask sequence that satisfies
$\varphi_{\texttt{global}}$.
Then, the probability that Algorithm~\ref{algo:subtask-sampling}
finds a feasible subtask sequence converges to one as the number of
sampling iterations approaches infinity.
\end{restatable}

\vspace{-3mm}
\subsection{Policy Learning}
\label{sec:policy-learning}
\vspace{-1mm}

The goal is to learn a decentralized goal-conditioned policy $\pi_i(a_i\mid o_i,\sigma_i)$ for each agent $i$, which takes as input its local observation $o_i\in\mathcal O_i$ and current reach-avoid subtask $\sigma_i$.
We assume that all agents are homogeneous.
The challenge is that the same set of local subtasks can correspond to different coordination requirements under different specifications.
For example, two agents may have the same local subtasks of reaching their respective target regions, while one specification allows them to complete these subtasks independently and another requires them to do so simultaneously.
Thus, the global coordination requirement cannot be determined from the local subtasks alone.
Consequently, commonly adopted cooperative MARL schemes based on shared team rewards and centralized critics~\citep{sunehag2017value,foerster2018counterfactual,rashid2020monotonic,yu2022surprising} are not well suited to our setting.
Instead, we use a subtask-dependent reward and cost function evaluated with respect to each agent's own subtask, and follow an IPPO-style approach~\citep{de2020independent} to learn decentralized policies and critics.

During training, the high-level planner described in Section~\ref{sec:task-decomposition-allocation} is not required.
We independently sample a reach-avoid subtask for each agent and train the low-level policies to complete their respective subtasks.
Although the subtasks are sampled independently, agents still interact and must account for one another, allowing the policies to learn interaction-aware behaviors; for example, multiple agents may sample subtasks that require reaching the same region and must avoid collisions with other agents.
Each agent learns to satisfy its subtask by receiving positive rewards for subtask satisfaction while being penalized for collisions and subtask violations.
High-level coordination requirements, such as synchronized or ordered subtask execution, are instead handled at test time by the centralized planner, which controls each agent's current subtask and its progression.
The learned subtask-conditioned policies are then reused to execute these planned subtasks across different specifications, enabling zero-shot generalization to unseen specifications without retraining.
We instantiate the low-level policies by adapting GenZ-LTL~\citep{guo2026one} because it explicitly models safety constraints, which are common in our specifications, while our framework is also compatible with other goal-conditioned RL methods capable of learning reach-avoid policies~\citep{qiu2023instructing, yalcinkaya2024compositional}.
Further implementation details are provided in Appendix~\ref{app:implementation}.

\vspace{-1mm}
\section{Experiments}
\label{sec:exps}
\vspace{-1mm}

In this section, we evaluate our method, SpecMAGIC, in a navigation environment across a wide range of multi-agent task specifications to answer the following questions: 
\textbf{Q1:} Can SpecMAGIC generalize zero-shot to unseen task specifications?
\textbf{Q2:} How well does SpecMAGIC scale as the number of agents increases?
\textbf{Q3:} What is the impact of the proposed sampling-based planner on SpecMAGIC versus a reference product-automaton-based planner?

\textbf{Environments.}
We consider a multi-agent \texttt{Zone} environment based on prior work~\citep{jackermeier2025deepltl, guo2026one, guo2026specrlbench}, where multiple agents with continuous action spaces and high-dimensional lidar observations navigate among randomly positioned colored regions to satisfy the given task specifications.
The \texttt{Zone} environment is similar to the commonly used Multi-Agent Particle Environment (MPE) for MARL~\citep{lowe2017multi}, but involves complex long-horizon tasks, including reaching sequences of regions with varying lengths, satisfying diverse coordination constraints among agents, and avoiding inter-agent collisions at all times.
An illustration of the environment with an example trajectory is shown in Figure~\ref{fig:envs-trajs}; more details are provided in Appendix~\ref{app:envs}.

\textbf{Baselines.}
We compare our method with ACC-MARL~\citep{yalcinkaya2025automata}, the only existing method that, to our knowledge, targets generalization across temporal-logic specifications in MARL. 
ACC-MARL first learns embeddings of compositions of deterministic finite automata (DFAs).
It then learns decentralized policies, with each agent's policy conditioned on the global state information together with the DFA embeddings of the specifications assigned to all agents.
Further details of the implementation are provided in Appendix~\ref{app:implementation}.

\textbf{LTL specifications.} 
For individual tasks $\Phi_{\texttt{ind}}$, we consider reach-avoid specifications of varying complexity, constructed from nested properties of the form $\neg\phi \until \varphi$ and combinations with $\mathsf{F}$, $\land$, and $\lor$.
For multi-agent tasks, we consider specifications that capture both independent and joint behaviors, including inter-agent dependency requirements $\Phi_{\texttt{dep}}$ and count requirements $\Phi_{\texttt{cnt}}$, both of which are constructed to involve all agents in the team rather than a subset of agents.
We then evaluate four types of specifications: $\Phi_{\texttt{ind}}$, $\Phi_{\texttt{ind/dep}}$, $\Phi_{\texttt{ind/cnt}}$, and $\Phi_{\texttt{ind/dep/cnt}}$.
The specifications used for evaluation are listed in Table~\ref{tab:complex-specs}.

\textbf{Evaluation metrics.}
For each LTL specification, we measure the success rate $\eta_s$ and violation rate $\eta_v$, corresponding to the fractions of trajectories that satisfy and violate the specification, respectively. 
The remaining trajectories are captured by $\eta_o = 1-(\eta_s+\eta_v)$. 
We also report the average number of steps $\mu$ required for successful trajectories to satisfy the specification, treating it as a secondary metric since efficiency is less informative when few trajectories succeed.

\begin{table*}[t]
\centering

\begin{minipage}[t]{0.71\textwidth}
\vspace{0pt}
\centering

\caption{
Evaluation results of finite-horizon tasks: success rate $\eta_s$,
violation rate $\eta_v$, and average steps $\mu$ to satisfy the complex
specifications listed in Table~\ref{tab:complex-specs}.
$\uparrow$: higher is better; $\downarrow$: lower is better.
\textbf{Bold}: the best performance for each metric.
Each value is reported as mean $\pm$ standard deviation over 5 seeds,
with 100 trajectories evaluated per seed.
}
\vspace{-3mm}
\label{tab:main-results-finite}

\begingroup
\setlength{\tabcolsep}{3pt}
\renewcommand{\arraystretch}{1.2}

\resizebox{\linewidth}{!}{
\begin{tabular}{cccccccc}
\toprule
\multicolumn{2}{c}{\multirow{2}{*}{Specifications}}
& \multicolumn{3}{c}{ACC-MARL}
& \multicolumn{3}{c}{Ours} \\
\multicolumn{2}{c}{}
& $\eta_s \uparrow$
& $\eta_v \downarrow$
& $\mu \downarrow$
& $\eta_s \uparrow$
& $\eta_v \downarrow$
& $\mu \downarrow$ \\ \midrule

\multirow{5}{*}{IND}
& $\Phi_{\texttt{ind}}$
& $0.21_{\pm0.09}$ & $0.66_{\pm0.08}$ & $510.29_{\pm105.01}$
& $\mathbf{0.95_{\pm0.02}}$
& $\mathbf{0.04_{\pm0.02}}$
& $\mathbf{290.33_{\pm36.64}}$ \\

& $\Phi_{\texttt{ind/dep}}$
& \multicolumn{3}{c}{\multirow{3}{*}{N/A}}
& $\mathbf{0.94_{\pm0.04}}$
& $\mathbf{0.05_{\pm0.03}}$
& $\mathbf{379.73_{\pm60.20}}$ \\

& $\Phi_{\texttt{ind/cnt}}$
& \multicolumn{3}{c}{}
& $\mathbf{0.95_{\pm0.02}}$
& $\mathbf{0.05_{\pm0.02}}$
& $\mathbf{380.11_{\pm49.64}}$ \\

& $\Phi_{\texttt{ind/dep/cnt}}$
& \multicolumn{3}{c}{}
& $\mathbf{0.91_{\pm0.04}}$
& $\mathbf{0.07_{\pm0.04}}$
& $\mathbf{460.72_{\pm75.88}}$ \\

& Average
& $0.21_{\pm0.09}$ & $0.66_{\pm0.08}$ & $510.29_{\pm105.01}$
& $\mathbf{0.94_{\pm0.03}}$
& $\mathbf{0.05_{\pm0.03}}$
& $\mathbf{377.72_{\pm82.62}}$ \\ \midrule

\multirow{5}{*}{OOD}
& $\Phi_{\texttt{ind}}$
& $0.04_{\pm0.02}$ & $0.74_{\pm0.09}$ & $832.44_{\pm23.00}$
& $\mathbf{0.96_{\pm0.02}}$
& $\mathbf{0.04_{\pm0.02}}$
& $\mathbf{379.79_{\pm33.89}}$ \\

& $\Phi_{\texttt{ind/dep}}$
& \multicolumn{3}{c}{\multirow{3}{*}{N/A}}
& $\mathbf{0.84_{\pm0.04}}$
& $\mathbf{0.10_{\pm0.04}}$
& $\mathbf{633.58_{\pm47.89}}$ \\

& $\Phi_{\texttt{ind/cnt}}$
& \multicolumn{3}{c}{}
& $\mathbf{0.94_{\pm0.04}}$
& $\mathbf{0.06_{\pm0.04}}$
& $\mathbf{433.42_{\pm37.13}}$ \\

& $\Phi_{\texttt{ind/dep/cnt}}$
& \multicolumn{3}{c}{}
& $\mathbf{0.83_{\pm0.05}}$
& $\mathbf{0.08_{\pm0.04}}$
& $\mathbf{659.28_{\pm43.10}}$ \\

& Average
& $0.04_{\pm0.02}$ & $0.74_{\pm0.09}$ & $832.44_{\pm23.00}$
& $\mathbf{0.89_{\pm0.07}}$
& $\mathbf{0.07_{\pm0.04}}$
& $\mathbf{526.52_{\pm129.06}}$ \\

\bottomrule
\end{tabular}
}

\endgroup
\end{minipage}
\hfill
\begin{minipage}[t]{0.27\textwidth}
\vspace{1mm}
\centering

\includegraphics[
    width=\linewidth,
    trim={45 0 20 85},
    clip
]{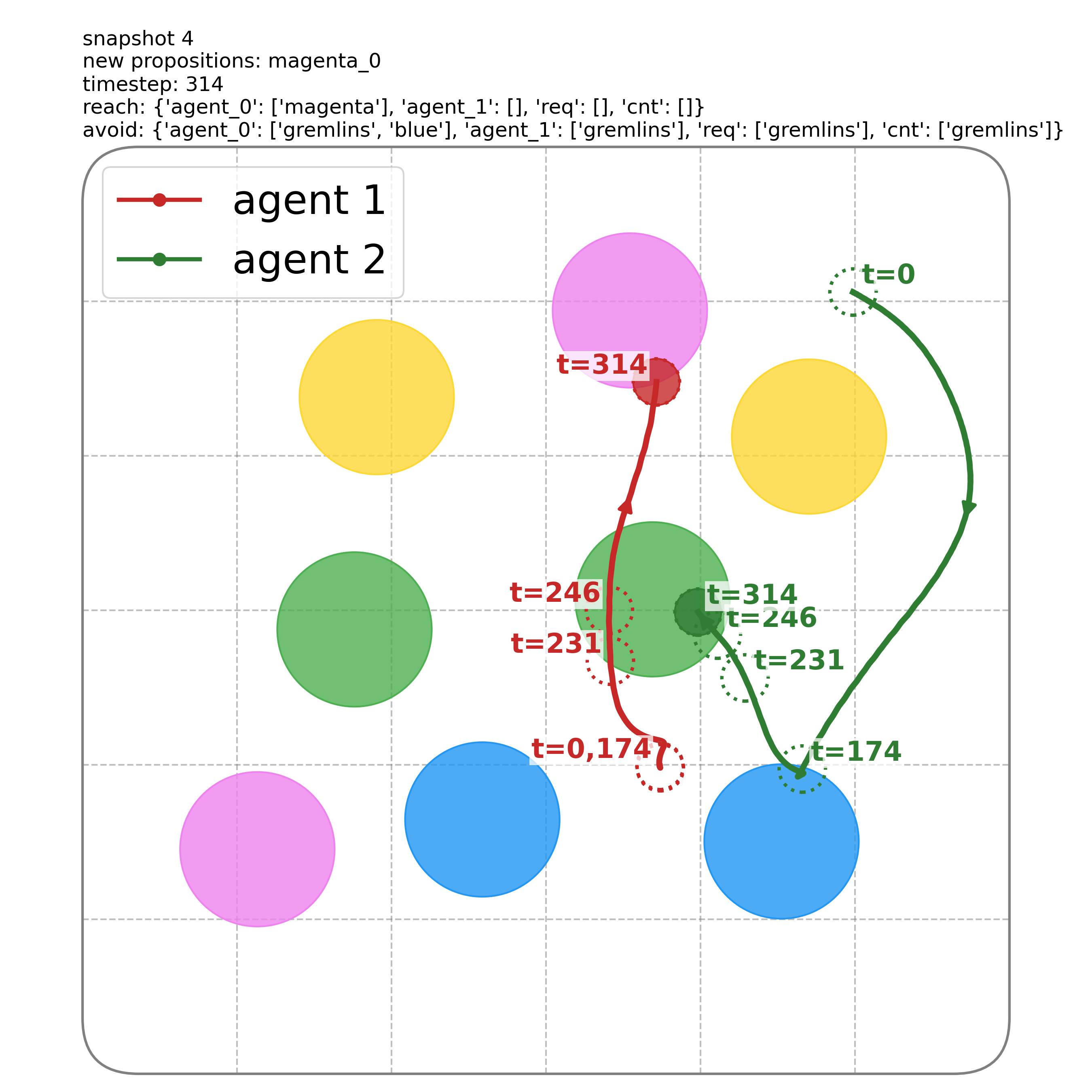}
\vspace{-6.5mm}
\captionof{figure}{
    Example trajectories of 2 agents in the \texttt{Zone} environment:
    \vspace{-7mm}
}
\[
    \begin{aligned}
    & \varphi_1: \neg \mathsf{b}\until\mathsf{m}, \varphi_2: \neg \mathsf{y}\until\mathsf{b}, \\[-1mm]
    & \varphi_{\texttt{dep}}: \mathsf{m}_1\ \texttt{Req}\ \mathsf{b}_2, \\[-1mm]
    & \varphi_{\texttt{cnt}}: \mathsf{Sync}[\{1,2\},\mathsf{g},2].
    \end{aligned}
    \]
\label{fig:envs-trajs}
\end{minipage}
\vspace{-3mm}
\end{table*}

\vspace{-1mm}
\subsection{Can SpecMAGIC generalize zero-shot to unseen task specifications (Q1)?}
\label{sec:zero-shot-generalization}
\vspace{-1mm}

\textbf{Finite-horizon tasks.} 
We first evaluate finite-horizon tasks, i.e., specifications that can be satisfied by finite-length trajectories.
We consider 3 agents, each assigned an individual task, together with additional inter-agent dependency $\Phi_{\texttt{dep}}$ and count requirements $\Phi_{\texttt{cnt}}$.
We evaluate both in-distribution (IND) and out-of-distribution (OOD) specifications, following the specification distribution used to train ACC-MARL.
ACC-MARL is trained on specifications represented as DFAs and sampled from a fixed distribution; IND specifications are drawn from this distribution, whereas OOD specifications fall outside it and require longer subtask sequences to satisfy.
The same IND/OOD specification sets used to evaluate ACC-MARL are used to evaluate SpecMAGIC.

As shown in Table~\ref{tab:main-results-finite}, our method consistently outperforms the baseline across all evaluated specifications, achieving higher success rates, lower violation rates, and better efficiency, as indicated by fewer steps required for satisfaction.
ACC-MARL supports only $\Phi_{\texttt{ind}}$, where each agent has an individual task and access to the task automata of the other agents.
Extending it to $\Phi_{\texttt{dep}}$ and $\Phi_{\texttt{cnt}}$ would require additional training distributions over these specifications, potentially biasing zero-shot evaluation.
ACC-MARL also assumes access to agent state information $s_i$, whereas agents in our environment receive only partial observations $o_i$ through LiDAR sensing.
Moreover, ACC-MARL pre-trains automaton embeddings using automata sampled from a fixed distribution, which limits its generalization when encountering unseen automata.
In contrast, our method decomposes specifications into reach-avoid subtasks and learns reusable subtask-conditioned policies, enabling unseen specifications to be handled by composing previously learned subtask behaviors and thereby improving OOD generalization.
Finally, performance generally decreases with increasing specification complexity, such as from $\Phi_{\texttt{ind}}$ to $\Phi_{\texttt{ind/dep/cnt}}$. 
Inter-agent dependency and count requirements introduce additional coordination constraints, while OOD specifications further increase the difficulty by requiring each agent to complete more subtasks and imposing more constraints in $\Phi_{\texttt{dep}}$ and $\Phi_{\texttt{cnt}}$. 
These factors make successful coordination more challenging and can increase the likelihood of inter-agent collisions in practice.
Nevertheless, our method maintains high success rates and low violation rates across these more complex and unseen specifications.

\textbf{Infinite-horizon tasks.}
Our method can also handle infinite-horizon specifications by repeatedly applying the same planning and execution procedure.
ACC-MARL is excluded from this evaluation since it relies on DFAs, which cannot represent such specifications.
We therefore report only the performance of our method, with the results provided in Appendix~\ref{app:additional-results} as a reference for future comparisons when suitable baselines become available.

\vspace{-1mm}
\subsection{How well does SpecMAGIC scale as the number of agents increases? (Q2)}
\label{sec:scale-agent-number}
\vspace{-1mm}

We evaluate scalability with respect to the number of agents using specifications from $\Phi_{\texttt{ind/dep/cnt}}$, with $N\in\{3,5,10\}$.
Note that the complexity of $\Phi_{\texttt{dep}}$ and $\Phi_{\texttt{cnt}}$ increases with $N$, since these specifications impose coordination requirements involving an increasing number of agents: more $\texttt{Req}$ operators are introduced in $\Phi_{\texttt{dep}}$ to ensure that each agent has at least one dependency with another agent, while the agent index sets of the count operators grow accordingly to include all agents (see Appendix~\ref{app:specs} for details).
We first evaluate the scalability of the high-level planner separately and then the overall framework.

\begin{wraptable}{r}{0.42\linewidth}
\centering
\vspace{-4.5mm}
\caption{
High-level planner scalability.
We report the product-automaton size $|\mathcal{B}_{\varphi_{\texttt{global}}}|$, the constructed tree size $|\mathcal{T}|$, sampling iterations, and planning time $t_{\texttt{path}}$, averaged over 10 seeds.
}
\label{tab:planner-scalability}
\Large
\renewcommand{\arraystretch}{1.1}
\vspace{-3mm}
\resizebox{1.0\linewidth}{!}{
\begin{tabular}{ccccc}
\toprule
$N$ & $|\mathcal{B}_{\varphi_{\texttt{global}}}|$ & $|\mathcal{T}|$ & Iterations         & $t_{\texttt{path}}$ \\ \midrule
3   & $98.7_{\pm42.5}$              & $7.9_{\pm3.3}$  & $64.7_{\pm100.5}$  & $0.02_{\pm0.02}$    \\
5   & $1789.3_{\pm1267.3}$          & $10.4_{\pm5.2}$ & $320.4_{\pm511.7}$ & $0.17_{\pm0.24}$    \\
10  & $>118814.3$                   & $12.6_{\pm5.2}$ & $477.4_{\pm530.2}$ & $8.64_{\pm0.77}$    \\ \bottomrule
\end{tabular}
}
\vspace{-3mm}
\end{wraptable}

\textbf{Evaluation of the high-level planner.}
We measure specification complexity by the number of states in the automaton corresponding to the global specification $\varphi_{\texttt{global}}$; this automaton is used only as a reference measure and is not required by our planner.
For the proposed sampling-based planner, we report the number of sampling iterations required to find a feasible path and the resulting tree size $|\mathcal{T}|$, which counts all explored nodes, including those not on the final feasible path.
We also report the average planning time $t_{\texttt{path}}$, excluding automaton-construction time; $t_{\texttt{path}}$ includes both finding feasible paths for the component specifications and the subsequent sampling process for the final multi-agent subtask allocation.
As shown in Table~\ref{tab:planner-scalability}, the automaton size grows rapidly with the number of agents and becomes intractable to construct for 10 agents (see Appendix~\ref{app:implementation} for details), while the sampling-based planner maintains a relatively small tree, with only a moderate increase in $|\mathcal{T}|$ and sampling iterations.

\begin{figure}[ht]
    \centering
    \includegraphics[width=0.8\linewidth]{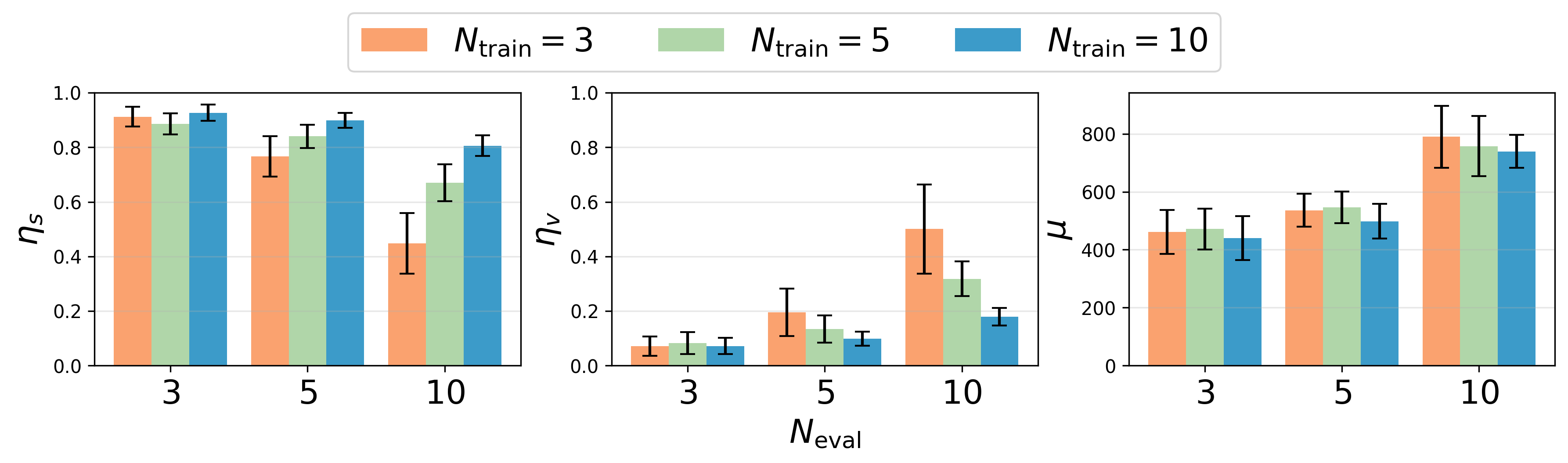}
    \vspace{-3mm}
    \caption{
    Performance under different numbers of agents during training and evaluation.
    We report the average success rate $\eta_s$, violation rate $\eta_v$, and number of steps $\mu$ to specification satisfaction for $N_{\texttt{train}},N_{\texttt{eval}}\in\{3,5,10\}$.
    Error bars indicate one standard deviation.
    }
    \label{fig:scalability-agent-number}
    \vspace{-3mm}
\end{figure}

\textbf{Evaluation of the overall framework.}
We next evaluate the overall framework by training policies with $N_{\texttt{train}}\in\{3,5,10\}$ agents and evaluating each policy with $N_{\texttt{eval}}\in\{3,5,10\}$ agents without retraining.
We report the success rate $\eta_s$, violation rate $\eta_v$, and average number of steps $\mu$ to satisfy the specifications, with results shown in Figure~\ref{fig:scalability-agent-number}.
When $N_{\texttt{eval}}=N_{\texttt{train}}$, performance slightly decreases as the number of agents increases, as indicated by lower success rates and higher violation rates, due to more frequent inter-agent interactions and the increasing complexity of $\Phi_{\texttt{dep}}$ and $\Phi_{\texttt{cnt}}$.
When $N_{\texttt{eval}}>N_{\texttt{train}}$, the performance drop becomes more noticeable, as the larger number of agents makes the environment more congested and shifts the observation distribution, making safety constraints more difficult to satisfy.
For a fixed $N_{\texttt{eval}}$, policies trained with more agents tend to perform better; for example, when $N_{\texttt{eval}}=10$, the policy trained with $N_{\texttt{train}}=5$ outperforms that trained with $N_{\texttt{train}}=3$.
In contrast, when $N_{\texttt{eval}}<N_{\texttt{train}}$, the policies maintain or slightly improve their performance, as training with more agents exposes the policies to more complex inter-agent interactions and leads to more robust execution when fewer agents are present.
Overall, these results show that our framework can be directly applied to different numbers of agents without retraining, while improving the robustness of the low-level policies could further enhance performance when scaling to larger multi-agent systems.

\vspace{-1mm}
\subsection{Ablation study: What is the impact of the proposed sampling-based planner on SpecMAGIC versus a reference search-based planner? (Q3)}
\label{sec:abltion-study}
\vspace{-1mm}

\vspace{-1mm}
\begin{figure}[ht]
    \centering
    \includegraphics[width=1.0\linewidth]{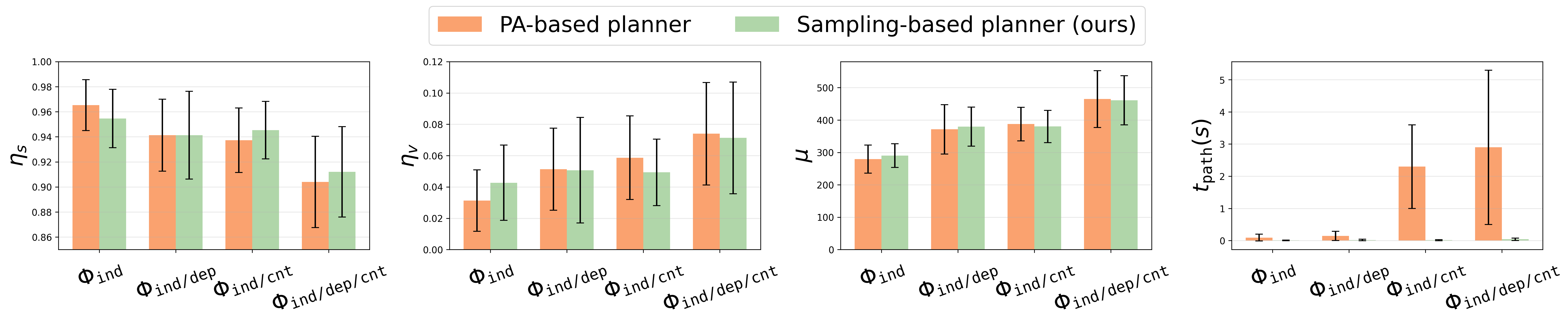}
    \vspace{-7mm}
    \caption{
    Average success rate $\eta_s$, violation rate $\eta_v$, and average number of steps $\mu$ to satisfy the IND specifications $\Phi_{\texttt{ind}}$, $\Phi_{\texttt{ind/dep}}$, $\Phi_{\texttt{ind/cnt}}$, and $\Phi_{\texttt{ind/dep/cnt}}$ listed in Table~\ref{tab:complex-specs}.
    }
    \label{fig:sampling-bfs-comparison}
\end{figure}
\vspace{-1mm}

We compare the proposed sampling-based planner with a product-automaton (PA)-based planner, which explicitly constructs the product automaton of the component specifications as discussed in Section~\ref{sec:task-decomposition-allocation} and performs graph search to find a feasible subtask sequence.
The resulting subtasks are executed sequentially using the same trained low-level policies as our sampling-based planner.
We evaluate both approaches on the 3-agent IND specifications in Table~\ref{tab:complex-specs}, with results shown in Figure~\ref{fig:sampling-bfs-comparison}.
We report the average planning time $t_{\texttt{path}}$ as defined in Section~\ref{sec:scale-agent-number}, excluding automaton-construction time for both methods.
Our sampling-based method achieves similar success rates, violation rates, and average steps to specification satisfaction as the search-based planner across all evaluated specifications, empirically validating the proposed sampling technique.
Meanwhile, our method requires substantially less planning time than the search-based planner, with the gap becoming more pronounced as specification complexity increases.
For specifications involving four or more agents evaluated in Section~\ref{sec:scale-agent-number}, the search-based planner fails to find a feasible path within a 10-minute time limit; thus, these results are excluded from the comparison.

\vspace{-1mm}
\section{Conclusion and Future Work}
\label{sec:conclusion}
\vspace{-1mm}

This work develops a hierarchical framework for zero-shot generalization across diverse LTL specifications in multi-agent systems.
The framework consists of a novel centralized sampling-based high-level planner with soundness and probabilistic completeness guarantees for planning over the reach-avoid subtask space, and decentralized subtask-conditioned policies for low-level execution.
Empirical results demonstrate that our method achieves strong generalization across diverse finite- and infinite-horizon specifications and remains effective as the number of agents increases.
A limitation of our framework is that the high-level planner remains centralized, which may limit scalability as the number of agents further increases.
Future work will investigate fully decentralized planning mechanisms to further improve scalability and integrate alternative safe RL methods~\citep{kushwaha2025survey} into the low-level policies that may better handle safety constraints.

\section*{Acknowledgement}
This work was supported in part by the U.S. National Science Foundation under grant CCF-2340776.

\subsection*{AI use statement}
In this work, we used generative AI tools for grammar checking and assistance with parts of the code.
All AI-assisted code was reviewed, verified, and tested for correctness by two authors.
We have reviewed all AI-assisted work and take responsibility for the final content of this work, including text, claims, and artifacts produced with the aid of generative AI.

\subsection*{Ethics statement}
This work does not introduce specific ethical concerns beyond those generally associated with autonomous multi-agent systems.
However, care should be taken when designing task specifications to ensure that they do not induce unsafe or undesirable behaviors.

\subsection*{Reproducibility statement}
To support reproducibility, we provide the source code and implementation details necessary to reproduce the experiments and results reported in this work in the supplementary materials.

\bibliography{iclr2027_conference}
\bibliographystyle{iclr2027_conference}

\appendix

\section{Theoretical Analysis}
\label{app:theorems}

\textbf{Soundness and (probabilistic) completeness.}
We first establish soundness, showing that any path returned by the proposed planner satisfies the global specification.

\soundness*

\begin{proof}
\vspace{-2mm}
    Let $\tau=[\sigma_0,\sigma_1,\ldots]$ be a path returned by Algorithm~\ref{algo:subtask-sampling}, constructed from the selected feasible subtask sequences $g_i^{j_i}\in\mathcal G_i$, $i\in\mathcal I$.
    By construction, the subtasks $\sigma$ in $\tau$ preserve their original order in each $g_i^{j_i}$, and each subtask in $\tau$ is compatible with its subsequent subtask.
    Therefore, $\tau\models\varphi_i$ for all $i\in\mathcal I$, and consequently, $\tau\models\bigwedge_{i\in\mathcal I}\varphi_i=\varphi_{\texttt{global}}$.
    \vspace{-2mm}
\end{proof}

Having established that every returned path satisfies the global specification, we next consider whether the planner can find such a path whenever one exists.
Specifically, we show probabilistic completeness: if a feasible subtask sequence exists, the probability of finding one converges to one as the number of sampling iterations increases.

\probcomplete*

\begin{proof}
    \vspace{-2mm}
    Let $M$ denote the maximum number of distinct tree nodes and $C$ the total number of possible subtask combinations across all specifications.
    Since $\{\mathcal G_i\}_{i\in\mathcal I}$ is finite, both $M$ and $C$ are finite constants determined by the specifications.
    Fix one feasible path, and suppose that a node along this path has been added to the tree. 
    By Algorithm~\ref{algo:subtask-sampling}, at the current iteration, this node is sampled with probability at least $1/M$, and a particular next subtask combination is sampled with probability $1/C$. 
    If this combination is compatible, the corresponding successor node is added to the tree. 
    Hence, the probability of making the corresponding successful extension in one iteration is at least $p=1/(MC)>0$, regardless of the current tree.
    Let $m$ denote the number of successful extensions required to reach a terminal node, and let $X_K$ denote the number of successful extensions along the fixed path after $K$ iterations.
    Since each iteration extends the fixed path with probability at least $p$ regardless of the current tree, we have $
    \Pr[X_K<m]
    \leq
    \sum_{j=0}^{m-1}\binom{K}{j}p^j(1-p)^{K-j}
    \leq
    \sum_{j=0}^{m-1}\binom{K}{m-1}p^j(1-p)^{K-j}
    \leq
    \binom{K}{m-1}\sum_{j=0}^{m-1}(1-p)^K
    \leq
    \binom{K}{m-1}m e^{-pK}
    \leq
    \frac{m}{(m-1)!}K^m e^{-pK}
    \rightarrow 0$ as $K\rightarrow\infty$, where the transitions rely on (1) $m \ll K$, (2) $p<\frac{1}{2}$, and (3) $(1-p) \leq e^{-p}$.
    Therefore, the probability that Algorithm~\ref{algo:subtask-sampling} fails to reach a terminal node converges to zero, and hence Algorithm~\ref{algo:subtask-sampling} is probabilistically complete.
\end{proof}

\section{Experimental Settings}
\label{app:exps}

\subsection{Environments}
\label{app:envs}

We extend the $\texttt{Zone}$ environment used in prior work~\citep{jackermeier2025deepltl, guo2026one, guo2026specrlbench} to a multi-agent setting.
The environment contains colored regions associated with the base atomic propositions ${\mathsf{blue},\mathsf{green},\mathsf{magenta},\mathsf{yellow}}$.
Each colored region appears twice in the environment.
For the 10-agent setting, we use a slightly larger environment and increase the radius of the colored regions to accommodate the larger number of agents.
For agent $i$, the corresponding atomic propositions are indexed as $AP_i=\{\mathsf{blue}_i,\mathsf{green}_i,\mathsf{magenta}_i,\mathsf{yellow}_i\}$, representing agent $i$ reaching the corresponding colored regions, and the overall proposition set is given by $AP=\bigcup_i AP_i$.
Following prior work~\citep{vaezipoor2021ltl2action, qiu2023instructing, jackermeier2025deepltl, guo2026one}, the colored regions do not overlap, so at most one proposition of each agent holds at any time.
Each agent is modeled as a \texttt{point} robot with continuous acceleration and steering actions.
We modify the observation space to include LiDAR measurements of neighboring agents in addition to the colored regions and ego-centric state information from onboard sensors, enabling interactions among multiple agents.
The initial positions of all agents and colored regions are randomly sampled at the beginning of each episode.
We construct environments with $N\in\{3,5,10\}$ agents to evaluate scalability with respect to the number of agents.
The maximum episode length is $T=2000$.
Illustrations of the resulting environments with 3, 5, and 10 agents are shown in Figure~\ref{fig:envs}.

\vspace{-1mm}
\begin{figure}[ht]
    \centering
    \begin{subfigure}{0.33\textwidth}
        \centering
        {\includegraphics[width=\textwidth]{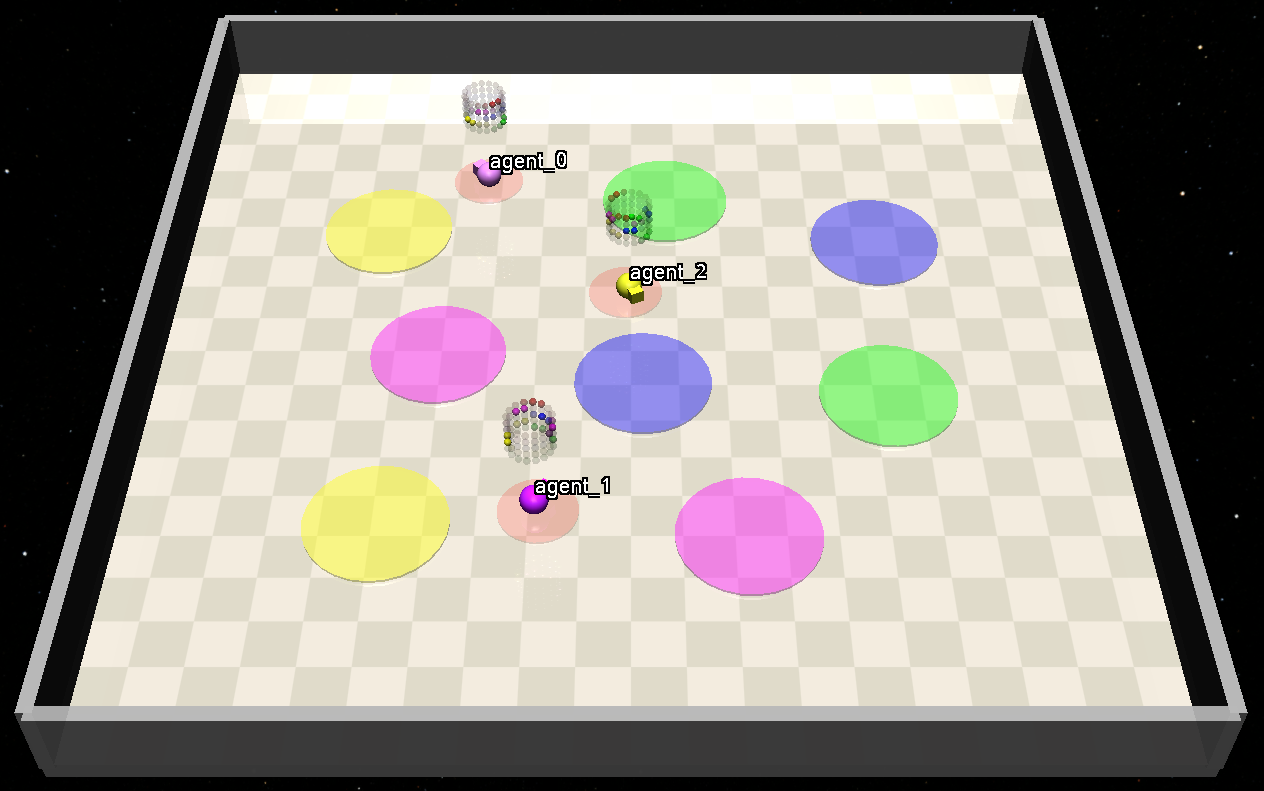}}
    \end{subfigure}
    \begin{subfigure}{0.315\textwidth}
        \centering
        \includegraphics[width=\textwidth]{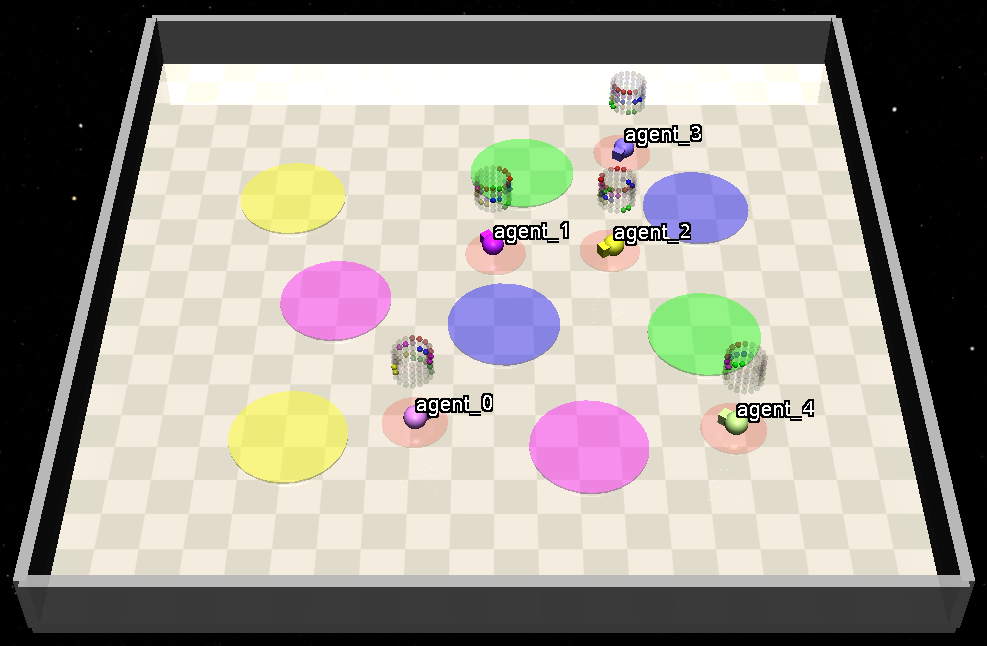}
    \end{subfigure}
    \begin{subfigure}{0.335\textwidth}
        \centering
        \includegraphics[width=\textwidth]{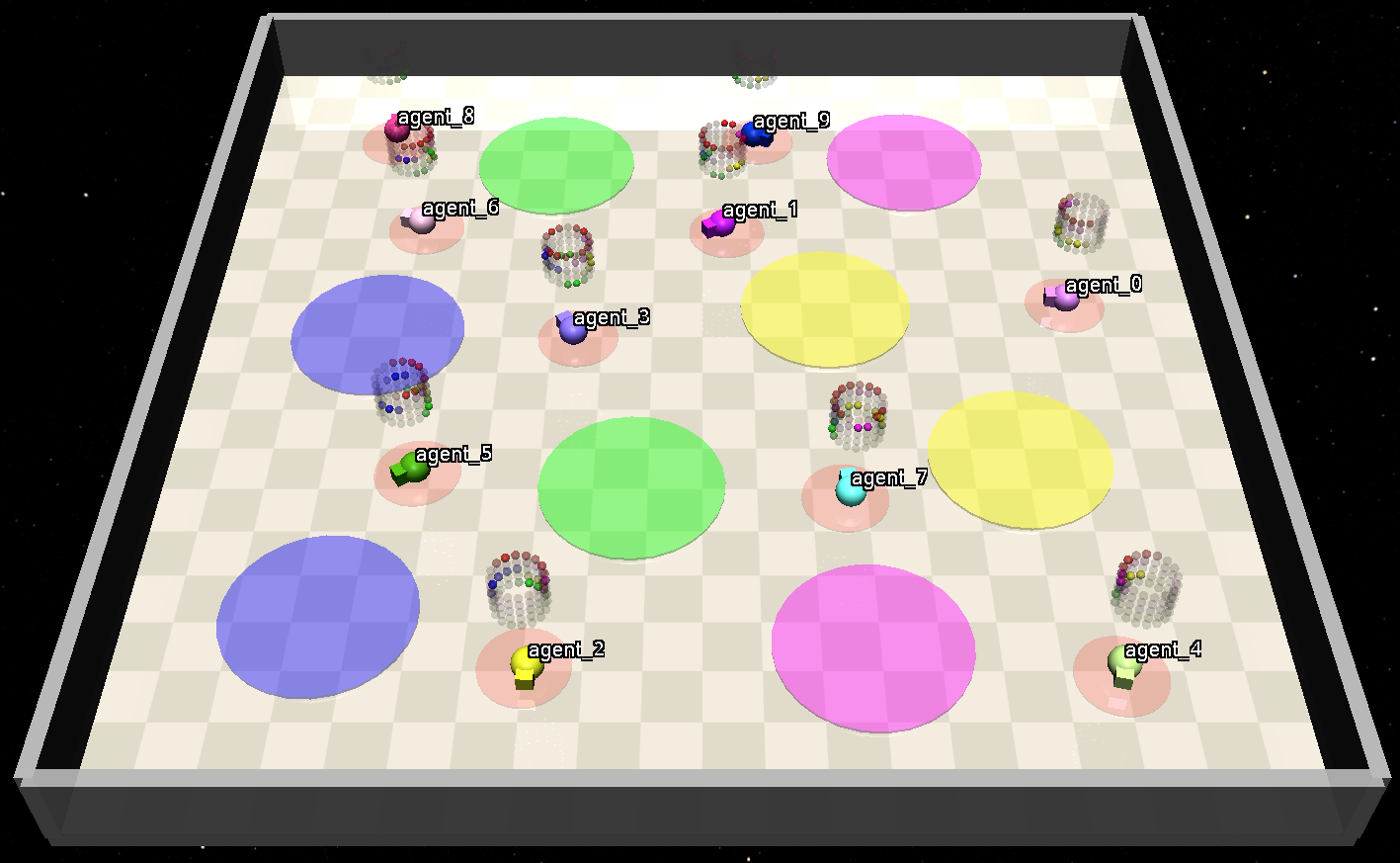}
    \end{subfigure}
    \caption{Visualizations of the \texttt{Zone} environments with $N=3$, $5$, and $10$ agents, from left to right.
    The positions of agents and zones are randomly sampled each episode.
    }
    \label{fig:envs}
\end{figure}

\subsection{Specifications}
\label{app:specs}

\textbf{Finite- and infinite-horizon tasks.}
We evaluate the methods on specifications with different types and levels of complexity, as summarized in Table~\ref{tab:complex-specs}. 
While many more specifications can be constructed, exhaustive evaluation is intractable since the space of possible specifications is infinite.
We therefore select a representative set of specifications covering different types and levels of complexity.
For each agent, we first define a set of single-agent $\texttt{sa}$ specifications that describe its individual task.
We then define inter-agent requirement $\texttt{dep}$ specifications to impose temporal dependencies between agents and count $\texttt{cnt}$ specifications to impose synchronous or asynchronous constraints on the number of agents satisfying a proposition. 
Based on these specifications, we construct four classes of multi-agent tasks: $\Phi_{\mathrm{indep}}$ contains only independent single-agent specifications, $\Phi_{\mathrm{indep/dep}}$ additionally includes requirement constraints, $\Phi_{\mathrm{indep/cnt}}$ includes count constraints, and $\Phi_{\mathrm{indep/dep/cnt}}$ combines all three types.
In comparison, the OOD specifications are structurally more complex, involving longer temporal dependencies and more compositional constraints than those encountered during training. 

\textbf{Scalability experiments.}
For the scalability experiments with different numbers of agents, we randomly select single-agent specifications from $\varphi_1$-$\varphi_9$ for each agent and construct new \texttt{dep} and \texttt{cnt} specifications according to the number of agents, such that every agent is involved in the inter-agent constraints. 
We then combine the single-agent, \texttt{dep}, and \texttt{cnt} specifications to form $\Phi_{\mathrm{indep/dep/cnt}}$ for evaluation. 
We also retrain the policies for each $N_{\mathrm{train}}\in\{3,5,10\}$ and evaluate each trained policy with $N_{\mathrm{eval}}\in\{3,5,10\}$ agents.
This procedure allows us to evaluate the scalability of the methods as the number of agents and the resulting specification complexity increase.

\textbf{Ablation study of the search-based method.}
For the ablation study, we evaluate the three IND global specifications in $\Phi_{\mathrm{indep/dep/cnt}}$ for the three-agent setting, i.e., $\varphi_1 \land \varphi_2 \land \varphi_3 \land \varphi_{10} \land \varphi_{13}$, $\varphi_4 \land \varphi_5 \land \varphi_6 \land \varphi_{11} \land \varphi_{14}$, and $\varphi_7 \land \varphi_8 \land \varphi_9 \land \varphi_{12} \land \varphi_{15}$.
The resulting automata contain an average of $99.67 \pm 34.72$ states and $1851.00 \pm 908.17$ transitions, while the corresponding individual specifications have an average of $6.53 \pm 3.30$ states and $21.47 \pm 19.23$ transitions.
We use these specifications to compare the proposed sampling-based method with the search-based method and demonstrate that sampling provides a more practical approach for handling complex global specifications with large automata.

\textbf{Translation of \texttt{dep} and \texttt{cnt} into LTL formulas.}
The proposed \texttt{dep} and \texttt{cnt} operators provide compact notation for inter-agent coordination requirements and can be translated into standard LTL formulas.
The operator $\alpha\,\texttt{Req}\,\beta$ requires $\beta$ to occur before the first occurrence of $\alpha$ and can be translated as $\neg\alpha\until\beta$.
The operator $\alpha\,\texttt{EachReq}\,\beta$ requires this prerequisite relation to hold for every non-consecutive occurrence of $\alpha$ and can be translated as $(\neg\alpha \until \beta)\land\always\left((\alpha \land \mathsf{X}\neg\alpha)\rightarrow\mathsf{X}(\neg\alpha\ \mathsf{W}\ \beta)\right)$, where $\mathsf{X}\ \varphi$ means that $\varphi$ holds at the next time step, and weak until is defined as $\varphi_1 \ \mathsf{W}\ \varphi_2 \equiv (\varphi_1 \until \varphi_2) \lor \always \varphi_1$~\citep{baier2008principles}.
For the \texttt{cnt} operators, let $\mathcal{I}$ denote the set of involved agents, $\alpha_i$ the proposition associated with agent $i\in\mathcal{I}$, and $m$ the required number of agents.
$\texttt{Sync}[\mathcal{I},\alpha,m]$ requires at least $m$ agents in $\mathcal{I}$ to satisfy their corresponding $\alpha_i$ simultaneously, and can be translated as the disjunction over all subsets $\mathcal{K}\subseteq\mathcal{I}$ with $|\mathcal{K}|=m$ of $\event(\bigwedge_{i\in\mathcal{K}}\alpha_i)$.
Similarly, $\texttt{Async}[\mathcal{I},\alpha,m]$ requires at least $m$ agents to each satisfy their corresponding $\alpha_i$ without requiring simultaneous satisfaction, and can be translated as the disjunction over all such subsets of $\bigwedge_{i\in\mathcal{K}}\event\alpha_i$.
Boolean combinations of \texttt{dep} and \texttt{cnt} operators are translated compositionally using standard LTL connectives.
Thus, these operators do not extend the expressiveness of LTL, but provide convenient notation for expressing inter-agent dependency and count requirements.

\begin{table}[ht]
\centering
\caption{
Specifications used for evaluation.
$\mathsf{b}$, $\mathsf{g}$, $\mathsf{m}$, and $\mathsf{y}$ denote $\mathsf{blue}$, $\mathsf{green}$, $\mathsf{magenta}$, and $\mathsf{yellow}$, respectively.
$\mathsf{sa}$ and $\mathsf{ma}$ denote specifications for single-agent and multi-agent behaviors, respectively. 
Top: finite-horizon specifications with 3 agents, with results in Table~\ref{tab:main-results-finite}. 
Bottom left: finite-horizon specifications with 5 and 10 agents for the scalability study, with results in Figure~\ref{fig:scalability-agent-number}. 
Bottom right: infinite-horizon specifications with 3 agents, with results in Table~\ref{tab:main-results-infinite}.
}
\vspace{1mm}
\label{tab:complex-specs}
\LARGE
\renewcommand{\arraystretch}{1.35}
\resizebox{1.0\linewidth}{!}{
\begin{tabular}{ccclcl}
\Xhline{1.8pt}
\multicolumn{2}{c}{Specification types}                                                      & \multicolumn{2}{c}{IND}                                                                                                                                                                                                                                                & \multicolumn{2}{c}{OOD}                                                                                                                                                                                                                                                         \\ \midrule
\multirow{19}{*}{\begin{tabular}[c]{@{}c@{}}Finite-\\ horizon\\ (3 agents)\end{tabular}} & $\texttt{sa}$  & $\varphi_{1}$                   & $(\neg \mathsf{m}\ \mathsf{U}\ (\mathsf{b} \lor \mathsf{g})) \land (\ \mathsf{F}\ \mathsf{m})$                                                                                                                                       & $\varphi_{16}$                  & $\neg \mathsf{m}\ \mathsf{U}\ (\mathsf{b} \land \ \mathsf{F}\ (\neg \mathsf{y}\ \mathsf{U}\ (\mathsf{g} \land \ \mathsf{F}\ \mathsf{m})))$                                                                                                    \\
                                                                            & $\texttt{sa}$  & $\varphi_{2}$                   & $(\ \mathsf{F}\ \mathsf{b}) \land (\neg \mathsf{b}\ \mathsf{U}\ (\mathsf{g} \land \ \mathsf{F}\ \mathsf{y}))$                                                                                                                        & $\varphi_{17}$                  & $\neg \mathsf{b}\ \mathsf{U}\ (\mathsf{g} \land \ \mathsf{F}\ (\neg \mathsf{m}\ \mathsf{U}\ (\mathsf{y} \land \ \mathsf{F}\ \mathsf{b})))$                                                                                                    \\
                                                                            & $\texttt{sa}$  & $\varphi_{3}$                   & $\neg (\mathsf{m} \lor \mathsf{y})\ \mathsf{U}\ (\mathsf{b} \land \ \mathsf{F}\ \mathsf{g})$                                                                                                                                         & $\varphi_{18}$                  & $\neg (\mathsf{m} \lor \mathsf{y})\ \mathsf{U}\ (\mathsf{b} \land \ \mathsf{F}\ (\neg \mathsf{b}\ \mathsf{U}\ (\mathsf{g} \land \ \mathsf{F}\ \mathsf{m})))$                                                                                  \\
                                                                            & $\texttt{sa}$  & $\varphi_{4}$                   & $\neg \mathsf{m}\ \mathsf{U}\ ((\mathsf{b} \lor \mathsf{y}) \land \ \mathsf{F}\ \mathsf{g})$                                                                                                                                         & $\varphi_{19}$                  & $\neg \mathsf{m}\ \mathsf{U}\ ((\mathsf{b} \lor \mathsf{y}) \land \ \mathsf{F}\ (\neg \mathsf{b}\ \mathsf{U}\ (\mathsf{g} \land \ \mathsf{F}\ \mathsf{m})))$                                                                                  \\
                                                                            & $\texttt{sa}$  & $\varphi_{5}$                   & $\neg (\mathsf{b} \lor \mathsf{y})\ \mathsf{U}\ (\mathsf{g} \land \ \mathsf{F}\ \mathsf{m})$                                                                                                                                         & $\varphi_{20}$                  & $\neg (\mathsf{b} \lor \mathsf{y})\ \mathsf{U}\ (\mathsf{g} \land \ \mathsf{F}\ (\neg \mathsf{g}\ \mathsf{U}\ (\mathsf{m} \land \ \mathsf{F}\ \mathsf{y})))$                                                                                  \\
                                                                            & $\texttt{sa}$  & $\varphi_{6}$                   & $(\ \mathsf{F}\ \mathsf{m}) \land (\neg \mathsf{y}\ \mathsf{U}\ (\mathsf{b} \lor \mathsf{g}))$                                                                                                                                       & $\varphi_{21}$                  & $\neg \mathsf{y}\ \mathsf{U}\ ((\mathsf{b} \lor \mathsf{g}) \land \ \mathsf{F}\ (\neg \mathsf{m}\ \mathsf{U}\ (\mathsf{y} \land \ \mathsf{F}\ \mathsf{b})))$                                                                                  \\
                                                                            & $\texttt{sa}$  & $\varphi_{7}$                   & $\neg \mathsf{g}\ \mathsf{U}\ (\mathsf{y} \land \ \mathsf{F}\ (\mathsf{b} \lor \mathsf{m}))$                                                                                                                                         & $\varphi_{22}$                  & $\neg \mathsf{g}\ \mathsf{U}\ (\mathsf{y} \land \ \mathsf{F}\ (\neg (\mathsf{y} \lor \mathsf{m})\ \mathsf{U}\ (\mathsf{b} \land \ \mathsf{F}\ \mathsf{m})))$                                                                                  \\
                                                                            & $\texttt{sa}$  & $\varphi_{8}$                   & $\neg \mathsf{m}\ \mathsf{U}\ ((\mathsf{g} \lor \mathsf{b}) \land \ \mathsf{F}\ \mathsf{y})$                                                                                                                                         & $\varphi_{23}$                  & $\neg \mathsf{m}\ \mathsf{U}\ ((\mathsf{g} \lor \mathsf{b}) \land \ \mathsf{F}\ (\neg \mathsf{y}\ \mathsf{U}\ (\mathsf{m} \land \ \mathsf{F}\ \mathsf{g})))$                                                                                  \\
                                                                            & $\texttt{sa}$  & $\varphi_{9}$                   & $\neg (\mathsf{g} \lor \mathsf{y})\ \mathsf{U}\ (\mathsf{m} \land \ \mathsf{F}\ \mathsf{b})$                                                                                                                                         & $\varphi_{24}$                  & $\neg (\mathsf{g} \lor \mathsf{y})\ \mathsf{U}\ (\mathsf{m} \land \ \mathsf{F}\ (\neg \mathsf{m}\ \mathsf{U}\ (\mathsf{b} \land \ \mathsf{F}\ \mathsf{y})))$                                                                                  \\
                                                                            & $\texttt{dep}$ & $\varphi_{10}$                  & $(\mathsf{b}_{0}\ \mathsf{Req}\ \mathsf{g}_{1}) \land (\mathsf{b}_{1}\ \mathsf{Req}\ \mathsf{g}_{2}) \land (\mathsf{b}_{2}\ \mathsf{Req}\ \mathsf{m}_{0})$                                                                           & $\varphi_{25}$                  & $(\mathsf{b}_{0}\ \mathsf{Req}\ \mathsf{g}_{1}) \land (\mathsf{b}_{0}\ \mathsf{Req}\ \mathsf{b}_{2}) \land (\mathsf{g}_{1}\ \mathsf{Req}\ \mathsf{g}_{0}) \land (\mathsf{b}_{2}\ \mathsf{Req}\ \mathsf{y}_{1})$                               \\
                                                                            & $\texttt{dep}$ & $\varphi_{11}$                  & $(\mathsf{g}_{0}\ \mathsf{Req}\ \mathsf{m}_{1}) \land (\mathsf{m}_{1}\ \mathsf{Req}\ \mathsf{b}_{2}) \land (\mathsf{m}_{2}\ \mathsf{Req}\ \mathsf{g}_{1})$                                                                           & $\varphi_{26}$                  & $(\mathsf{g}_{1}\ \mathsf{Req}\ \mathsf{g}_{0}) \land (\mathsf{g}_{1}\ \mathsf{Req}\ \mathsf{y}_{2}) \land (\mathsf{g}_{0}\ \mathsf{Req}\ \mathsf{m}_{1}) \land (\mathsf{y}_{2}\ \mathsf{Req}\ \mathsf{m}_{0})$                               \\
                                                                            & $\texttt{dep}$ & $\varphi_{12}$                  & $(\mathsf{b}_{0}\ \mathsf{Req}\ \mathsf{g}_{1}) \land (\mathsf{y}_{1}\ \mathsf{Req}\ \mathsf{m}_{2}) \land (\mathsf{b}_{2}\ \mathsf{Req}\ \mathsf{y}_{0})$                                                                           & $\varphi_{27}$                  & $(\mathsf{y}_{0}\ \mathsf{Req}\ \mathsf{m}_{2}) \land (\mathsf{y}_{0}\ \mathsf{Req}\ \mathsf{m}_{1}) \land (\mathsf{m}_{2}\ \mathsf{Req}\ \mathsf{b}_{0}) \land (\mathsf{m}_{1}\ \mathsf{Req}\ \mathsf{b}_{2})$                               \\
                                                                            & $\texttt{cnt}$ & $\varphi_{13}$                  & $\ \mathsf{Sync}\ [\{0,2\}, \mathsf{y}, 2] \land \ \mathsf{Async}\ [\{0,1,2\}, \mathsf{g}, 3]$                                                                                                                                       & $\varphi_{28}$                  & $\ \mathsf{Sync}\ [\{0,1,2\}, \mathsf{g}, 2] \land \ \mathsf{Async}\ [\{0,1\}, \mathsf{y}, 1] \land \ \mathsf{Async}\ [\{1,2\}, \mathsf{m}, 1]$                                                                                               \\
                                                                            & $\texttt{cnt}$ & $\varphi_{14}$                  & $\ \mathsf{Sync}\ [\{0,1,2\}, \mathsf{m}, 2] \land \ \mathsf{Async}\ [\{0,2\}, \mathsf{y}, 2]$                                                                                                                                       & $\varphi_{29}$                  & $\ \mathsf{Async}\ [\{0,1,2\}, \mathsf{m}, 3] \land \ \mathsf{Sync}\ [\{0,2\}, \mathsf{y}, 1] \land \ \mathsf{Sync}\ [\{0,1,2\}, \mathsf{b}, 2]$                                                                                              \\
                                                                            & $\texttt{cnt}$ & $\varphi_{15}$                  & $\ \mathsf{Sync}\ [\{1,2\}, \mathsf{g}, 2] \land \ \mathsf{Async}\ [\{0,1,2\}, \mathsf{b}, 2]$                                                                                                                                       & $\varphi_{30}$                  & $\ \mathsf{Sync}\ [\{1,2\}, \mathsf{g}, 1] \land \ \mathsf{Async}\ [\{0,1,2\}, \mathsf{b}, 2] \land \ \mathsf{Sync}\ [\{0,1\}, \mathsf{m}, 1]$                                                                                                \\
                                                                            & $\texttt{ma}$  & $\Phi_{\mathrm{indep}}$         & $\{\{\varphi_{1}, \varphi_{2}, \varphi_{3}\}, \{\varphi_{4}, \varphi_{5}, \varphi_{6}\}, \{\varphi_{7}, \varphi_{8}, \varphi_{9}\}\}$                                                                                     & $\Phi_{\mathrm{indep}}$         & $\{\{\varphi_{16}, \varphi_{17}, \varphi_{18}\}, \{\varphi_{19}, \varphi_{20}, \varphi_{21}\}, \{\varphi_{22}, \varphi_{23}, \varphi_{24}\}\}$                                                                                     \\
                                                                            & $\texttt{ma}$  & $\Phi_{\mathrm{indep/dep}}$     & $\{\{\varphi_{1}, \varphi_{2}, \varphi_{3}, \varphi_{10}\}, \{\varphi_{4}, \varphi_{5}, \varphi_{6}, \varphi_{11}\}, \{\varphi_{7}, \varphi_{8}, \varphi_{9}, \varphi_{12}\}\}$                                           & $\Phi_{\mathrm{indep/dep}}$     & $\{\{\varphi_{16}, \varphi_{17}, \varphi_{18}, \varphi_{25}\}, \{\varphi_{19}, \varphi_{20}, \varphi_{21}, \varphi_{26}\}, \{\varphi_{22}, \varphi_{23}, \varphi_{24}, \varphi_{27}\}\}$                                           \\
                                                                            & $\texttt{ma}$  & $\Phi_{\mathrm{indep/cnt}}$     & $\{\{\varphi_{1}, \varphi_{2}, \varphi_{3}, \varphi_{13}\}, \{\varphi_{4}, \varphi_{5}, \varphi_{6}, \varphi_{14}\}, \{\varphi_{7}, \varphi_{8}, \varphi_{9}, \varphi_{15}\}\}$                                           & $\Phi_{\mathrm{indep/cnt}}$     & $\{\{\varphi_{16}, \varphi_{17}, \varphi_{18}, \varphi_{28}\}, \{\varphi_{19}, \varphi_{20}, \varphi_{21}, \varphi_{29}\}, \{\varphi_{22}, \varphi_{23}, \varphi_{24}, \varphi_{30}\}\}$                                           \\
                                                                            & $\texttt{ma}$  & $\Phi_{\mathrm{indep/dep/cnt}}$ & $\{\{\varphi_{1}, \varphi_{2}, \varphi_{3}, \varphi_{10}, \varphi_{13}\}, \{\varphi_{4}, \varphi_{5}, \varphi_{6}, \varphi_{11}, \varphi_{14}\}, \{\varphi_{7}, \varphi_{8}, \varphi_{9}, \varphi_{12}, \varphi_{15}\}\}$ & $\Phi_{\mathrm{indep/dep/cnt}}$ & $\{\{\varphi_{16}, \varphi_{17}, \varphi_{18}, \varphi_{25}, \varphi_{28}\}, \{\varphi_{19}, \varphi_{20}, \varphi_{21}, \varphi_{26}, \varphi_{29}\}, \{\varphi_{22}, \varphi_{23}, \varphi_{24}, \varphi_{27}, \varphi_{30}\}\}$ \\ \midrule
\end{tabular}

}

\begin{minipage}[t]{0.45\linewidth}
\centering

\renewcommand{\arraystretch}{1.225}
\resizebox{\linewidth}{!}{
\begin{tabular}{cccc}
\hline
\multirow{7}{*}{\begin{tabular}[c]{@{}c@{}}Finite-\\ horizon\\ (5 agents)\end{tabular}}  & $\texttt{dep}$ & $\varphi_{46}$                  & $(\mathsf{b}_{0}\ \mathsf{Req}\ \mathsf{g}_{1}) \land (\mathsf{b}_{1}\ \mathsf{Req}\ \mathsf{g}_{2}) \land (\mathsf{g}_{2}\ \mathsf{Req}\ \mathsf{y}_{3}) \land (\mathsf{g}_{3}\ \mathsf{Req}\ \mathsf{g}_{4})$                                                                                                                                                                                                                                                                                                                                                     \\
                                                                                         & $\texttt{dep}$ & $\varphi_{47}$                  & $(\mathsf{g}_{0}\ \mathsf{Req}\ \mathsf{m}_{1}) \land (\mathsf{b}_{1}\ \mathsf{Req}\ \mathsf{b}_{2}) \land (\mathsf{y}_{2}\ \mathsf{Req}\ \mathsf{g}_{3}) \land (\mathsf{b}_{3}\ \mathsf{Req}\ \mathsf{y}_{4})$                                                                                                                                                                                                                                                                                                                                                     \\
                                                                                         & $\texttt{dep}$ & $\varphi_{48}$                  & $(\mathsf{b}_{0}\ \mathsf{Req}\ \mathsf{b}_{1}) \land (\mathsf{g}_{1}\ \mathsf{Req}\ \mathsf{g}_{2}) \land (\mathsf{m}_{2}\ \mathsf{Req}\ \mathsf{y}_{3}) \land (\mathsf{y}_{3}\ \mathsf{Req}\ \mathsf{m}_{4})$                                                                                                                                                                                                                                                                                                                                                     \\
                                                                                         & $\texttt{cnt}$ & $\varphi_{49}$                  & $\mathsf{Async}\ [\{0,1,2,3,4\}, \mathsf{y}, 3] \land \mathsf{Sync}\ [\{0,1,2,3,4\}, \mathsf{b}, 2]$                                                                                                                                                                                                                                                                                                                                                                                                                                                                \\
                                                                                         & $\texttt{cnt}$ & $\varphi_{50}$                  & $\mathsf{Sync}\ [\{0,1,2,3,4\}, \mathsf{y}, 2] \land \mathsf{Async}\ [\{0,1,2,3,4\}, \mathsf{g}, 4]$                                                                                                                                                                                                                                                                                                                                                                                                                                                                \\
                                                                                         & $\texttt{cnt}$ & $\varphi_{51}$                  & $\mathsf{Sync}\ [\{0,1,2,3,4\}, \mathsf{m}, 2] \land \mathsf{Sync}\ [\{0,1,2,3,4\}, \mathsf{g}, 2]$                                                                                                                                                                                                                                                                                                                                                                                                                                                                 \\
                                                                                         & $\texttt{ma}$  & $\Phi_{\mathrm{indep/dep/cnt}}$ & \begin{tabular}[c]{@{}c@{}}$\{\{\varphi_{1}, \varphi_{2}, \varphi_{3}, \varphi_{4}, \varphi_{5}, \varphi_{46}, \varphi_{49}\},$ \\ $\{\varphi_{5}, \varphi_{6}, \varphi_{7}, \varphi_{8}, \varphi_{9}, \varphi_{47}, \varphi_{50}\},$ \\ $\{\varphi_{1}, \varphi_{3}, \varphi_{5}, \varphi_{7}, \varphi_{9}, \varphi_{48}, \varphi_{51}\}\}$\end{tabular}                                                                                                                                                                                                    \\ \hline
\multirow{7}{*}{\begin{tabular}[c]{@{}c@{}}Finite-\\ horizon\\ (10 agents)\end{tabular}} & $\texttt{dep}$ & $\varphi_{52}$                  & \begin{tabular}[c]{@{}c@{}}$(\mathsf{b}_{0}\ \mathsf{Req}\ \mathsf{g}_{1}) \land (\mathsf{b}_{1}\ \mathsf{Req}\ \mathsf{g}_{2}) \land (\mathsf{b}_{2}\ \mathsf{Req}\ \mathsf{m}_{3}) \land (\mathsf{b}_{3}\ \mathsf{Req}\ \mathsf{b}_{4}) \land (\mathsf{g}_{4}\ \mathsf{Req}\ \mathsf{m}_{5})$ \\ $\land (\mathsf{b}_{5}\ \mathsf{Req}\ \mathsf{g}_{6}) \land (\mathsf{b}_{6}\ \mathsf{Req}\ \mathsf{g}_{7}) \land (\mathsf{m}_{7}\ \mathsf{Req}\ \mathsf{b}_{8}) \land (\mathsf{y}_{8}\ \mathsf{Req}\ \mathsf{b}_{9})$\end{tabular}                                 \\
                                                                                         & $\texttt{dep}$ & $\varphi_{53}$                  & \begin{tabular}[c]{@{}c@{}}$(\mathsf{y}_{0}\ \mathsf{Req}\ \mathsf{g}_{1}) \land (\mathsf{y}_{1}\ \mathsf{Req}\ \mathsf{m}_{2}) \land (\mathsf{b}_{2}\ \mathsf{Req}\ \mathsf{b}_{3}) \land (\mathsf{g}_{3}\ \mathsf{Req}\ \mathsf{b}_{4}) \land (\mathsf{b}_{4}\ \mathsf{Req}\ \mathsf{b}_{5})$ \\ $\land (\mathsf{g}_{5}\ \mathsf{Req}\ \mathsf{g}_{6}) \land (\mathsf{y}_{6}\ \mathsf{Req}\ \mathsf{b}_{7}) \land (\mathsf{g}_{7}\ \mathsf{Req}\ \mathsf{g}_{8}) \land (\mathsf{m}_{8}\ \mathsf{Req}\ \mathsf{g}_{9})$\end{tabular}                                 \\
                                                                                         & $\texttt{dep}$ & $\varphi_{54}$                  & \begin{tabular}[c]{@{}c@{}}$(\mathsf{m}_{0}\ \mathsf{Req}\ \mathsf{b}_{1}) \land (\mathsf{m}_{1}\ \mathsf{Req}\ \mathsf{y}_{2}) \land (\mathsf{y}_{2}\ \mathsf{Req}\ \mathsf{b}_{3}) \land (\mathsf{b}_{3}\ \mathsf{Req}\ \mathsf{m}_{4}) \land (\mathsf{b}_{4}\ \mathsf{Req}\ \mathsf{y}_{5})$ \\ $\land (\mathsf{y}_{5}\ \mathsf{Req}\ \mathsf{m}_{6}) \land (\mathsf{m}_{6}\ \mathsf{Req}\ \mathsf{b}_{7}) \land (\mathsf{g}_{7}\ \mathsf{Req}\ \mathsf{y}_{8}) \land (\mathsf{y}_{8}\ \mathsf{Req}\ \mathsf{m}_{9})$\end{tabular}                                 \\
                                                                                         & $\texttt{cnt}$ & $\varphi_{55}$                  & $\mathsf{Async}\ [\{0,1,2,3,4,5,6,7,8,9\}, \mathsf{b}, 9]$                                                                                                                                                                                                                                                                                                                                                                                                                                                                                                           \\
                                                                                         & $\texttt{cnt}$ & $\varphi_{56}$                  & $\mathsf{Async}\ [\{0,1,2,3,4,5,6,7,8,9\}, \mathsf{m}, 7]$                                                                                                                                                                                                                                                                                                                                                                                                                                                                                                           \\
                                                                                         & $\texttt{cnt}$ & $\varphi_{57}$                  & $\mathsf{Async}\ [\{0,1,2,3,4,5,6,7,8,9\}, \mathsf{g}, 9]$                                                                                                                                                                                                                                                                                                                                                                                                                                                                                                           \\
                                                                                         & $\texttt{ma}$  & $\Phi_{\mathrm{indep/dep/cnt}}$ & \begin{tabular}[c]{@{}c@{}}$\{\{\varphi_{3}, \varphi_{8}, \varphi_{1}, \varphi_{6}, \varphi_{4}, \varphi_{9}, \varphi_{2}, \varphi_{5}, \varphi_{7}, \varphi_{3}, \varphi_{52}, \varphi_{55}\},$ \\ $\{\varphi_{7}, \varphi_{2}, \varphi_{9}, \varphi_{4}, \varphi_{1}, \varphi_{6}, \varphi_{8}, \varphi_{3}, \varphi_{5}, \varphi_{8}, \varphi_{53}, \varphi_{56}\},$ \\ $\{\varphi_{5}, \varphi_{1}, \varphi_{7}, \varphi_{3}, \varphi_{9}, \varphi_{2}, \varphi_{6}, \varphi_{4}, \varphi_{8}, \varphi_{1}, \varphi_{54}, \varphi_{57}\}\}$\end{tabular} \\ \Xhline{1.8pt}
\end{tabular}
}

\end{minipage}
\hfill
\begin{minipage}[t]{0.53\linewidth}
\centering

\renewcommand{\arraystretch}{1.2}
\resizebox{\linewidth}{!}{
\begin{tabular}{cccc}
\hline
\multirow{19}{*}{\begin{tabular}[c]{@{}c@{}}Infinite-\\ horizon\end{tabular}} & $\texttt{sa}$  & $\varphi_{31}$                  & $(\always\event \mathsf{b}) \land (\always\event \mathsf{g}) \land (\always \neg \mathsf{y})$                                                                                                                                                 \\
                                                                              & $\texttt{sa}$  & $\varphi_{32}$                  & $(\always\event \mathsf{y}) \land (\always\event \mathsf{g}) \land (\always \neg \mathsf{b})$                                                                                                                                                 \\
                                                                              & $\texttt{sa}$  & $\varphi_{33}$                  & $(\always\event \mathsf{m}) \land (\always\event \mathsf{b}) \land (\always \neg \mathsf{g})$                                                                                                                                                 \\
                                                                              & $\texttt{sa}$  & $\varphi_{34}$                  & $(\always\event \mathsf{y}) \land (\always\event \mathsf{m}) \land (\always \neg \mathsf{g})$                                                                                                                                                 \\
                                                                              & $\texttt{sa}$  & $\varphi_{35}$                  & $(\always\event \mathsf{b}) \land (\always\event \mathsf{m}) \land (\always \neg \mathsf{y})$                                                                                                                                                 \\
                                                                              & $\texttt{sa}$  & $\varphi_{36}$                  & $(\always\event \mathsf{g}) \land (\always\event \mathsf{y}) \land (\always \neg \mathsf{b})$                                                                                                                                                 \\
                                                                              & $\texttt{sa}$  & $\varphi_{37}$                  & $(\always\event \mathsf{g}) \land (\always\event \mathsf{b}) \land (\always \neg \mathsf{m})$                                                                                                                                                 \\
                                                                              & $\texttt{sa}$  & $\varphi_{38}$                  & $(\always\event \mathsf{m}) \land (\always\event \mathsf{y}) \land (\always \neg \mathsf{g})$                                                                                                                                                 \\
                                                                              & $\texttt{sa}$  & $\varphi_{39}$                  & $(\always\event \mathsf{y}) \land (\always\event \mathsf{b}) \land (\always \neg \mathsf{m})$                                                                                                                                                 \\
                                                                              & $\texttt{dep}$ & $\varphi_{40}$                  & $(\mathsf{b}_{0}\ \mathsf{EachReq}\ \mathsf{g}_{1}) \land (\mathsf{y}_{1}\ \mathsf{EachReq}\ \mathsf{b}_{2})$                                                                                                                                 \\
                                                                              & $\texttt{dep}$ & $\varphi_{41}$                  & $(\mathsf{y}_{0}\ \mathsf{EachReq}\ \mathsf{m}_{1}) \land (\mathsf{b}_{1}\ \mathsf{EachReq}\ \mathsf{y}_{2})$                                                                                                                                 \\
                                                                              & $\texttt{dep}$ & $\varphi_{42}$                  & $(\mathsf{g}_{0}\ \mathsf{EachReq}\ \mathsf{y}_{1}) \land (\mathsf{m}_{1}\ \mathsf{EachReq}\ \mathsf{b}_{2})$                                                                                                                                 \\
                                                                              & $\texttt{cnt}$ & $\varphi_{43}$                  & $\neg \ \mathsf{Async}\ [\{0,1,2\}, \mathsf{m}, 2]$                                                                                                                                                                                           \\
                                                                              & $\texttt{cnt}$ & $\varphi_{44}$                  & $\neg \ \mathsf{Sync}\ [\{0,1,2\}, \mathsf{b}, 2]$                                                                                                                                                                                            \\
                                                                              & $\texttt{cnt}$ & $\varphi_{45}$                  & $\neg \ \mathsf{Async}\ [\{0,1,2\}, \mathsf{g}, 2]$                                                                                                                                                                                           \\
                                                                              & $\texttt{ma}$  & $\Phi_{\mathrm{indep}}$         & $\{\{\varphi_{31}, \varphi_{32}, \varphi_{33}\}, \{\varphi_{34}, \varphi_{35}, \varphi_{36}\}, \{\varphi_{37}, \varphi_{38}, \varphi_{39}\}\}$                                                                                     \\
                                                                              & $\texttt{ma}$  & $\Phi_{\mathrm{indep/dep}}$     & $\{\{\varphi_{31}, \varphi_{32}, \varphi_{33}, \varphi_{40}\}, \{\varphi_{34}, \varphi_{35}, \varphi_{36}, \varphi_{41}\}, \{\varphi_{37}, \varphi_{38}, \varphi_{39}, \varphi_{42}\}\}$                                           \\
                                                                              & $\texttt{ma}$  & $\Phi_{\mathrm{indep/cnt}}$     & $\{\{\varphi_{31}, \varphi_{32}, \varphi_{33}, \varphi_{43}\}, \{\varphi_{34}, \varphi_{35}, \varphi_{36}, \varphi_{44}\}, \{\varphi_{37}, \varphi_{38}, \varphi_{39}, \varphi_{45}\}\}$                                           \\
                                                                              & $\texttt{ma}$  & $\Phi_{\mathrm{indep/dep/cnt}}$ & $\{\{\varphi_{31}, \varphi_{32}, \varphi_{33}, \varphi_{40}, \varphi_{43}\}, \{\varphi_{34}, \varphi_{35}, \varphi_{36}, \varphi_{41}, \varphi_{44}\}, \{\varphi_{37}, \varphi_{38}, \varphi_{39}, \varphi_{42}, \varphi_{45}\}\}$ \\ \Xhline{1.8pt}
\end{tabular}
}

\end{minipage}

\end{table}

\section{Implementation Details}
\label{app:implementation}

We use the official codebase of the baseline ACC-MARL\footnote{\url{https://github.com/rad-dfa/acc-marl}}. 
The low-level policy used in our framework is adapted from GenZ-LTL\footnote{\url{https://github.com/BU-DEPEND-Lab/GenZ-LTL}} by extending it to the multi-agent setting, while we implement the proposed high-level sampling-based planner. 
For a fair comparison, all methods use the same network architectures: the actor and critic are fully connected networks with three hidden layers of size $[64,64,64]$, and the value function network has two hidden layers of size $[64,64]$. 
We retain the original state/observation preprocessing of each method. 
All methods use the \texttt{Adam} optimizer with a learning rate of $3\times10^{-4}$ and are trained for 15M environment interactions. 
The discount factor is set to $\gamma=0.998$ for \texttt{ZoneEnv}. 
Additional hyperparameter details are provided in the code included in the supplementary materials. 
All methods are trained on a 16-core AMD CPU and an NVIDIA GeForce RTX 4090 GPU, with our model taking approximately 6 hours to train.

\textbf{Training and evaluation.}
For each agent $i$, we construct a feasible reach-avoid subtask set by considering every assignment $\alpha_i\in 2^{AP_i}$ as a candidate reach condition $A_i^+$.
For each candidate, assignments that conflict with $A_i^+$ are discarded, and all combinations of the remaining assignments are enumerated to construct the avoid condition $A_i^-$.
This yields the subtask set $\xi_i=\{(A_{i,m}^+,A_{i,m}^-)\}_{m=1}$.
During training, a reach-avoid subtask $\sigma_i$is uniformly sampled from the corresponding $\xi_i$ for each agent.
We define a binary reward $r_i=1$ when the current true propositions, obtained according to $L_i$, satisfy the reach condition of $\sigma_i$, and $r_i=0$ otherwise.
Similarly, we define a binary cost $c_i=1$ when the current true propositions satisfy the avoid condition of $\sigma_i$ or agent $i$ collides with another agent, and $c_i=0$ otherwise.
If any $c_i>0$, the episode is terminated immediately, since the corresponding specification can no longer be satisfied.
Otherwise, whenever $r_i=1$, a new subtask is re-sampled for agent $i$.
We train GenZ-LTL using the rewards/costs collected during interactions, and the policy and critics are shared across agents because the agents are homogeneous in the \texttt{Zone} environment.
The 5 seeds used for training and evaluation are $[1,11,22,33,44]$.
During evaluation, the agent and zone positions are randomly sampled for each episode, while the same sampled configurations are used across different methods under each seed for fair comparison.

\textbf{Specification satisfaction evaluation.}
For each rollout, we evaluate its trace against every specification component $\varphi_i$, $i\in\mathcal I$.
For finite-horizon tasks, a trajectory is considered successful only if all specification components are satisfied, whereas it is considered a violation if any component is violated.
For infinite-horizon tasks, violation is determined in the same way.
For each non-violating trajectory, we count the visits to accepting states for every specification component and take the minimum of these counts as the number of accepting-state visits reported in Table~\ref{tab:main-results-infinite}.

\textbf{Product automaton.}
We use the Spot\footnote{\url{https://spot.lre.epita.fr/}} library to construct an automaton for each specification component $\varphi_i$ and the corresponding product automaton for the global specification $\varphi_{\texttt{global}}=\bigwedge_{i\in\mathcal I}\varphi_i$.
We then record the number of states in the resulting global automaton.
For the specifications in Section~\ref{sec:scale-agent-number}, we perform this analysis for $\Phi_{\texttt{ind/dep/cnt}}$ with $N\in\{3,5,10\}$ and report the mean and standard deviation of the automaton-state counts in Table~\ref{tab:planner-scalability}.
For $N=10$, the complete product automata could not be constructed within tractable computational limits (1 hour limit), as each global specification contains 12 component specifications, including 10 individual-agent specifications together with one $\texttt{dep}$ and one $\texttt{cnt}$ specification with $|AP|=40$.
We therefore report the largest observed state number as a lower bound, denoted by $>{\rm num}$ in Table~\ref{tab:planner-scalability}.

\textbf{Details of the sampling-based planner.}
For each $\varphi_i$, $i \in \mathcal{I}$, let $\mathcal{G}_i=\{g_i^1,g_i^2,\ldots,g_i^{M_i}\}$ denote the set of feasible subtask sequences, where $M_i$ is the number of feasible subtask sequences for $\varphi_i$. 
Each $g_i^j\in\mathcal G_i$ is a sequence $g_i^j=[\sigma_{i,1}^j,\sigma_{i,2}^j,\ldots]$, where $\sigma_{i,k}^j=(A_{i,k}^{+,j},A_{i,k}^{-,j})$ denotes a reach-avoid subtask.
We assume without loss of generality that each reach condition $A_{i,k}^{+,j}$ and avoid condition $A_{i,k}^{-,j}$ is expressed in disjunctive normal form (DNF)~\citep{rosen1999discrete}, where each conjunctive clause is represented by the set of atomic propositions it contains; for example, $(a\land b)\lor(c\land d)$ is represented as $\{\{a,b\},\{c,d\}\}$.
Algorithm~\ref{algo:subtask-extend} extends a sampled node by finding compatible next subtasks.
We first identify the expandable components $\mathcal I_{\mathrm{exp}}$ such that, for each $i\in\mathcal I_{\mathrm{exp}}$, none of its subtask sequences has reached its end, and then sample a nonempty subset $\mathcal I_{\mathrm{s}}\subseteq\mathcal I_{\mathrm{exp}}$.
Now we explain the main loop of Algorithm~\ref{algo:subtask-extend} (line~\ref{algo-line:extend-loop}).
For each $i\in\mathcal I_{\mathrm{s}}$, if no subtask sequence has been selected, we uniformly sample $g_i^j\in\mathcal G_i$, retrieve its first subtask, and update the corresponding tuple to $(i,j,1)$.
Otherwise, given the tuple $(i,j,k)$ stored in $x_{\mathrm{parent}}$, we retrieve the next subtask $\sigma_{i,k+1}^j$ and update the corresponding tuple to $(i,j,k+1)$.

We then check whether the resulting node is valid by verifying the compatibility of the selected subtasks based on their reach and avoid conditions.
For each component $i\in\mathcal I_{\mathrm{s}}$ whose current subtask is $\sigma_{i,k}^j=(A_{i,k}^{+,j},A_{i,k}^{-,j})$, one conjunctive clause is sampled from $A_{i,k}^{+,j}$ and denoted by $\tilde A_i^+$.
Thus $\tilde A_i^+=\emptyset$ for $i\in\mathcal I\setminus\mathcal I_{\mathrm{s}}$ since no new reach condition is considered in the current extension.
The centralized planner then constructs the joint reach condition $A_{\mathrm{joint}}^+ := \bigcup_{i\in\mathcal I}\tilde A_i^+$.
To construct the joint avoid condition, we consider every component $i\in\mathcal I$.
For $i\in\mathcal I_{\mathrm{s}}$, we use the avoid condition of the newly retrieved subtask.
For $i\in\mathcal I\setminus\mathcal I_{\mathrm{s}}$, given its stored tuple $(i,j,k)$, if $j=0$, we uniformly sample a sequence $g_i^j\in\mathcal G_i$ and use the avoid condition of its first subtask; otherwise, we use the avoid condition of the current subtask $\sigma_{i,k}^j$.
Let $\tilde A_i^-$ denote the resulting avoid condition for component $i$.
Since each $\tilde A_i^-$ is represented in DNF as a set of conjunctive clauses, with each clause itself represented as a set of atomic propositions, the joint avoid condition is also a set of such clauses and is defined as
$A_{\mathrm{joint}}^-:=\bigcup_{i\in\mathcal I}\tilde A_i^-$.
The joint reach and avoid conditions are incompatible if any conjunctive clause $B\in A^-$ satisfies $B\subseteq A^+$, since satisfying the reach condition would necessarily violate the avoid condition.
Otherwise, propositions already contained in $A^+$ are removed from each avoid clause, yielding the final avoid condition $\hat A^-:=\{B\setminus A^+\mid B\in A^-\}$.
For example, if $A^+=\{a,b\}$ and $A^-=\{\{a,b,c\},\{d\}\}$, the conditions are compatible and reduce to $A^+=\{a,b\}$ and $\hat A^-=\{\{c\},\{d\}\}$.
In contrast, if $A^+=\{a,b\}$ and $A^-=\{\{a\},\{c\}\}$, they are incompatible because $\{a\}\subseteq A^+$.
This compatibility check depends only on the Boolean structure of the reach and avoid conditions.

\textbf{Replanning and value-guided sampling.}
During execution, we evaluate the current active propositions $L(s)$ and use them to update the automaton state of each specification component.
If none of the automaton states changes, the agents continue following the current plan.
Whenever the automaton state of at least one component changes, we recompute the feasible reach-avoid subtask sequences from the updated automaton states and invoke the planner again to generate a new plan from the current specification progress.
For infinite-horizon specifications, we construct each candidate subtask sequence from its finite prefix followed by $n_{\mathrm{cyc}}$ repetitions of its cyclic suffix, with $n_{\mathrm{cyc}}=2$ in all experiments, and apply the same planning and replanning procedure to these finite sequences.
On the other hand, at test time, learned value functions can in principle be used to guide subtask selection, e.g., by choosing the candidate subtask $\sigma$ with the largest value $V(s,\sigma)$ at the current state $s$.
However, evaluating a multi-step plan also requires estimating the values of subsequent subtasks at future states that have not yet been observed.
One possible approach is to train an additional model to predict a future state $\tilde{s}$ and then evaluate $V(\tilde{s},\sigma')$ for the next subtask $\sigma'$.
However, the planning horizon can be long, and prediction errors can accumulate over multiple subtasks, making accurate future-state estimates difficult to obtain.
Therefore, value-function-based lookahead becomes challenging for long-horizon task allocation considered in our setting.
Moreover, under partial observability, optimizing the planned subtask sequence is challenging, and value-guided sampling does not offer an obvious advantage.
We leave a systematic investigation of this direction to future work.

\section{Additional Results}
\label{app:additional-results}

\textbf{Performance on simple specifications.}
We additionally evaluate our method and ACC-MARL on simple specifications consisting of reach-only tasks, i.e., $\event \alpha$, and reach-avoid tasks, i.e., $\neg \alpha \until \beta$, where $\alpha,\beta\in AP$ and $\alpha\neq\beta$.
The results are shown in Table~\ref{tab:simple-specs}.
Note that for reach-only specifications, violations arise solely from inter-agent collisions.
ACC-MARL achieves substantially higher success rates on these simple specifications than on the more complex specifications considered in the main experiments, suggesting that its performance degradation is mainly due to the increased specification complexity.
In contrast, our method maintains consistently high success rates and low violation rates across both simple and complex specifications.

\begin{table}[ht]
\centering
\caption{
Evaluation results on simple specifications: success rate $\eta_s$, violation rate $\eta_v$, and average number of steps $\mu$.
$\uparrow$: higher is better; $\downarrow$: lower is better.
Each value is averaged over 5 seeds, with 100 trajectories per seed.
}
\label{tab:simple-specs}
\renewcommand{\arraystretch}{1.2}
\resizebox{0.6\linewidth}{!}{
\begin{tabular}{ccccc}
\toprule
Method                    & Specifications & $\eta_s \uparrow$ & $\eta_v \downarrow$ & $\mu \downarrow$    \\ \midrule
\multirow{2}{*}{ACC-MARL} & Reach-Only     & $0.77_{\pm0.08}$  & $0.19_{\pm0.07}$    & $261.27_{\pm39.58}$ \\
                          & Reach-Avoid    & $0.60_{\pm0.09}$  & $0.37_{\pm0.08}$    & $226.06_{\pm38.71}$ \\ \midrule
\multirow{2}{*}{Ours}     & Reach-Only     & $0.97_{\pm0.02}$  & $0.02_{\pm0.02}$    & $139.92_{\pm14.16}$ \\
                          & Reach-Avoid    & $0.98_{\pm0.02}$  & $0.01_{\pm0.01}$    & $148.40_{\pm15.47}$ \\ \bottomrule
\end{tabular}
}
\end{table}

\textbf{Infinite-horizon tasks.}
We also evaluate infinite-horizon tasks, i.e., specifications whose satisfaction requires infinite-length trajectories.
We focus on specifications involving $\always\event\varphi$, which induce continual navigation and inter-agent interactions, rather than $\event\always\varphi$, where agents remain in specific regions after reaching them.
ACC-MARL is the only existing baseline applicable to our setting, but it is excluded from this evaluation because DFAs cannot represent infinite-horizon specifications.
We therefore report the performance of our method without baseline comparison, providing a reference for future evaluation when suitable baselines become available.
The evaluated specifications are listed in Table~\ref{tab:complex-specs}, with results reported in Table~\ref{tab:main-results-infinite}.
We report the violation rate $\eta_v$ and the average number of visits to accepting states among non-violating trajectories, $\mu_{\texttt{acc}}$.
Our method naturally extends to this setting, as the same sampling-based planning procedure can be repeatedly applied to the subtask sequences obtained from each specification.
The learned policy is conditioned only on the current subtask and can therefore be directly reused throughout infinite-horizon execution.
Our method can obtain low violation rates and frequent visits to accepting states across all evaluated specifications, although $\mu_{\texttt{acc}}$ decreases and $\eta_v$ increases slightly as dependency and count requirements are introduced due to the increased coordination difficulty.

\begin{table*}[ht]
\centering
\caption{
Evaluation results of infinite-horizon tasks: violation rate $\eta_v$ and average number of visits to
accepting states $\mu_{\text{acc}}$ on infinite-horizon tasks listed in Table~\ref{tab:complex-specs}. 
$\uparrow$: higher is better; $\downarrow$: lower is better.
Each value is averaged over 5 seeds, with 100 trajectories per seed.
}
\vspace{-1mm}
\label{tab:main-results-infinite}
\renewcommand{\arraystretch}{1.1}
\resizebox{0.75\linewidth}{!}{
\begin{tabular}{cccccccc}
\Xhline{1.2pt}
Method                                                                          & Metrics                       & $\Phi_{\mathrm{ind}}$    & $\Phi_{\mathrm{ind/dep}}$ & $\Phi_{\mathrm{ind/cnt}}$ & $\Phi_{\mathrm{ind/dep/cnt}}$ & Average                   \\ \hline
\multirow{2}{*}{Ours}   & $\mu_{\texttt{acc}} \uparrow$ & $10.71_{\pm1.86}$ & $5.98_{\pm0.87}$   & $10.48_{\pm1.72}$  & $5.55_{\pm0.84}$       & $8.18_{\pm2.80}$ \\
                                                                                & $\eta_v \downarrow$           & $0.05_{\pm0.03}$           & $0.15_{\pm0.03}$   & $0.05_{\pm0.02}$   & $0.16_{\pm0.04}$                & $0.10_{\pm0.06}$ \\

                          \Xhline{1.2pt}
\end{tabular}
}
\end{table*}

\begin{figure}[ht]
    \centering
    \includegraphics[width=1.0\linewidth]{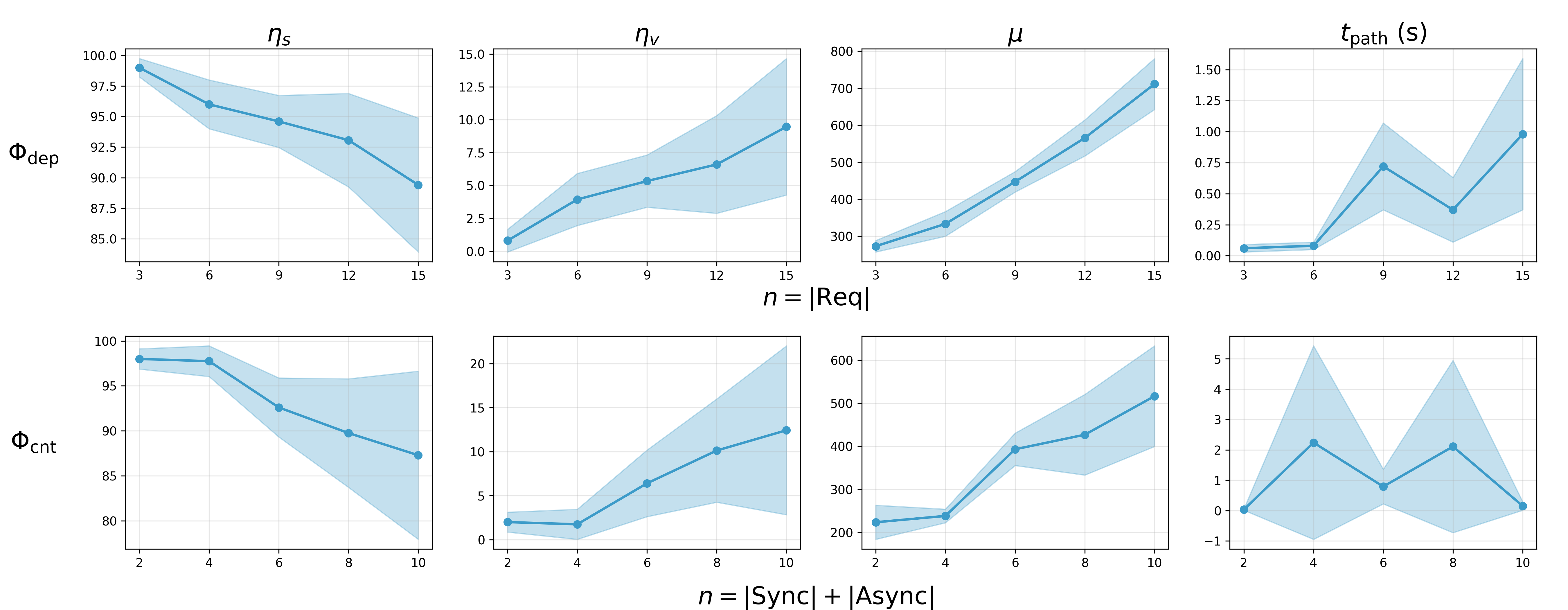}
    \caption{
    Performance as the complexity of inter-agent coordination specifications increases.
    The top row varies the number of dependency operators $n=|\texttt{Req}|$ in $\varphi_{\texttt{dep}}$, while the bottom row varies the number of count operators $n=|\texttt{Sync}|+|\texttt{Async}|$ in $\varphi_{\texttt{cnt}}$.
    We report the success rate $\eta_s$, violation rate $\eta_v$, average number of steps $\mu$ to satisfy a specification, and planning time $t_{\texttt{path}}$.
    Solid lines show the mean, and shaded regions indicate one standard deviation.
    }
    \label{fig:req_cnt_complecity}
\end{figure}

\textbf{Coordination-specification complexity.}
We also evaluate how performance changes as the complexity of the inter-agent coordination specifications increases in the 3-agent setting.
To isolate the effects of $\varphi_{\texttt{dep}}$ and $\varphi_{\texttt{cnt}}$, each agent is assigned a simple individual task of the form $\event\alpha$, while the coordination-specification complexity is varied separately.
For dependency requirements, we construct $\varphi_{\texttt{dep}}$ as a conjunction of $\texttt{Req}$ clauses and increase $n=|\texttt{Req}|$, introducing more inter-agent dependency relations.
Similarly, for count requirements, we construct $\varphi_{\texttt{cnt}}$ as a conjunction of $\texttt{Sync}$ and $\texttt{Async}$ clauses and increase $n=|\texttt{Sync}|+|\texttt{Async}|$, introducing more count requirements.
Although dependency and count requirements can be combined within the same global specification, we vary them separately here to isolate the effect of each type of coordination complexity.
As shown in Figure~\ref{fig:req_cnt_complecity}, increasing either type of specification complexity generally decreases the success rate $\eta_s$ and increases the violation rate $\eta_v$ and the number of steps $\mu$ required to satisfy the specification, as agents must satisfy a larger number of inter-agent coordination requirements.
Meanwhile, the planning time $t_{\texttt{path}}$ remains tractable as the number of coordination operators increases, since the sampling-based planner operates directly on subtask sequences without explicitly constructing the product automaton.

\textbf{Implicit coordination.}
Besides the coordination requirements explicitly specified by $\varphi_{\texttt{dep}}$ and $\varphi_{\texttt{cnt}}$, our method can also handle implicit coordination among agents.
To evaluate this, we use the TokenEnv environments of ACC-MARL~\citep{yalcinkaya2025automata} (Buttons-2, Rooms-2, Buttons-4, and Rooms-4), which require agents to coordinate implicitly, e.g., pressing a button to open a door for another agent.
We use the same DFA tasks as ACC-MARL, from which the reach-avoid subtask sequences are obtained.
Following ACC-MARL, the low-level policy of agent $i$ is $\pi_i(a_i \mid s_i, \boldsymbol{\sigma})$, where $s_i$ is the global state from agent $i$'s own point of view and $\boldsymbol{\sigma}=(\sigma_1,\ldots,\sigma_N)$ contains the current reach-avoid subtasks of all agents, in place of the DFA embeddings of all agents used by ACC-MARL.
We assume that the communication graph among agents is connected~\citep{zhu2024survey}, which is also implicitly assumed in ACC-MARL, since each policy takes the progressed DFAs of all agents as input.
In our environment, each agent sends its current true propositions to the centralized planner synchronously at each step, and the planner sends the planned subtask back to each agent.
In TokenEnv, each agent instead maintains a local copy of the planner, which receives the current true propositions of all agents synchronously at each step, from which it updates the automaton states of all component specifications before replanning.
As shown in Table~\ref{tab:tokenenv}, our method achieves comparable or higher success rates across all environments and task classes without defining additional propositions or specifying any coordination requirements, showing that it can also learn implicit coordination.
The improvement is more substantial for OOD tasks and for four agents, consistent with the trend in Table~\ref{tab:main-results-finite}.

\begin{table}[ht]
\centering
\caption{Evaluation results in the TokenEnv environments of ACC-MARL: success rate $\eta_s$ on the Reach, ReachAvoid, and RAD tasks and their out-of-distribution (OOD) variants. Bold: the best performance for each task class. Each value is reported as mean $\pm$ standard deviation over 5 seeds, with 1,000 episodes evaluated per seed.}
\label{tab:tokenenv}
\resizebox{\linewidth}{!}{
\begin{tabular}{cccccccc}
\toprule
Env & Method & Reach & ReachAvoid & RAD & Reach (OOD) & ReachAvoid (OOD) & RAD (OOD) \\
\midrule
\multirow{2}{*}{Buttons-2} & ACC-MARL & $0.860 \pm 0.042$ & $0.790 \pm 0.057$ & $0.821 \pm 0.044$ & $\mathbf{0.677 \pm 0.085}$ & $0.456 \pm 0.074$ & $0.522 \pm 0.046$ \\
 & Ours & $0.859 \pm 0.029$ & $\mathbf{0.805 \pm 0.041}$ & $\mathbf{0.854 \pm 0.022}$ & $0.617 \pm 0.050$ & $\mathbf{0.477 \pm 0.057}$ & $\mathbf{0.598 \pm 0.033}$ \\
\midrule
\multirow{2}{*}{Rooms-2} & ACC-MARL & $0.890 \pm 0.034$ & $0.871 \pm 0.040$ & $0.866 \pm 0.032$ & $0.723 \pm 0.082$ & $0.577 \pm 0.065$ & $0.559 \pm 0.039$ \\
 & Ours & $\mathbf{0.948 \pm 0.056}$ & $\mathbf{0.886 \pm 0.059}$ & $\mathbf{0.929 \pm 0.042}$ & $\mathbf{0.959 \pm 0.047}$ & $\mathbf{0.747 \pm 0.076}$ & $\mathbf{0.875 \pm 0.050}$ \\
\midrule
\multirow{2}{*}{Buttons-4} & ACC-MARL & $0.751 \pm 0.044$ & $0.676 \pm 0.032$ & $0.736 \pm 0.031$ & $0.482 \pm 0.029$ & $0.282 \pm 0.021$ & $0.355 \pm 0.005$ \\
 & Ours & $\mathbf{0.915 \pm 0.094}$ & $\mathbf{0.810 \pm 0.079}$ & $\mathbf{0.867 \pm 0.087}$ & $\mathbf{0.813 \pm 0.191}$ & $\mathbf{0.547 \pm 0.110}$ & $\mathbf{0.686 \pm 0.145}$ \\
\midrule
\multirow{2}{*}{Rooms-4} & ACC-MARL & $0.832 \pm 0.017$ & $0.764 \pm 0.018$ & $0.795 \pm 0.010$ & $0.539 \pm 0.058$ & $0.363 \pm 0.020$ & $0.392 \pm 0.012$ \\
 & Ours & $\mathbf{0.987 \pm 0.015}$ & $\mathbf{0.888 \pm 0.011}$ & $\mathbf{0.947 \pm 0.011}$ & $\mathbf{0.992 \pm 0.014}$ & $\mathbf{0.697 \pm 0.026}$ & $\mathbf{0.840 \pm 0.025}$ \\
\bottomrule
\end{tabular}}
\end{table}

\section{Qualitative Results}
\label{app:qualitative-results}

We conduct a qualitative analysis of our method using specification $\{\varphi_4,\varphi_5,\varphi_6,\varphi_{11},\varphi_{14}\}$, with the visualization of the trajectories shown in Figure~\ref{fig:qualitative-results} and the planning process of the high-level planner in Figure~\ref{fig:rrt_planner}.
In natural language, the specifications mean that $\texttt{agent}_0$ is required to avoid the $\mathsf{magenta}$ region until reaching either the $\mathsf{blue}$ or $\mathsf{yellow}$ region, followed by reaching the $\mathsf{green}$ region; $\texttt{agent}_1$ must avoid the blue and yellow regions until reaching the green region, followed by the magenta region; and $\texttt{agent}_2$ must reach the magenta region and avoid the yellow region until reaching either the blue or green region.
The $\texttt{dep}$ specification requires $\texttt{agent}_1$ to reach the magenta region before $\texttt{agent}_0$ reaches the green region, $\texttt{agent}_2$ to reach the blue region before $\texttt{agent}_1$ reaches the magenta region, and $\texttt{agent}_1$ to reach the green region before $\texttt{agent}_2$ reaches the magenta region.
The $\texttt{cnt}$ specification requires that at least two agents must reach the magenta region simultaneously, and both $\texttt{agent}_0$ and $\texttt{agent}_2$ must reach the yellow region, without requiring simultaneous visits.

\begin{figure}[ht]
    \centering
    \captionsetup[subfigure]{
        font=scriptsize,
        justification=centering,
        singlelinecheck=false
    }

    \begin{subfigure}{0.23\textwidth}
        \centering
        \includegraphics[width=\textwidth, trim={0 0 0 60}, clip]
        {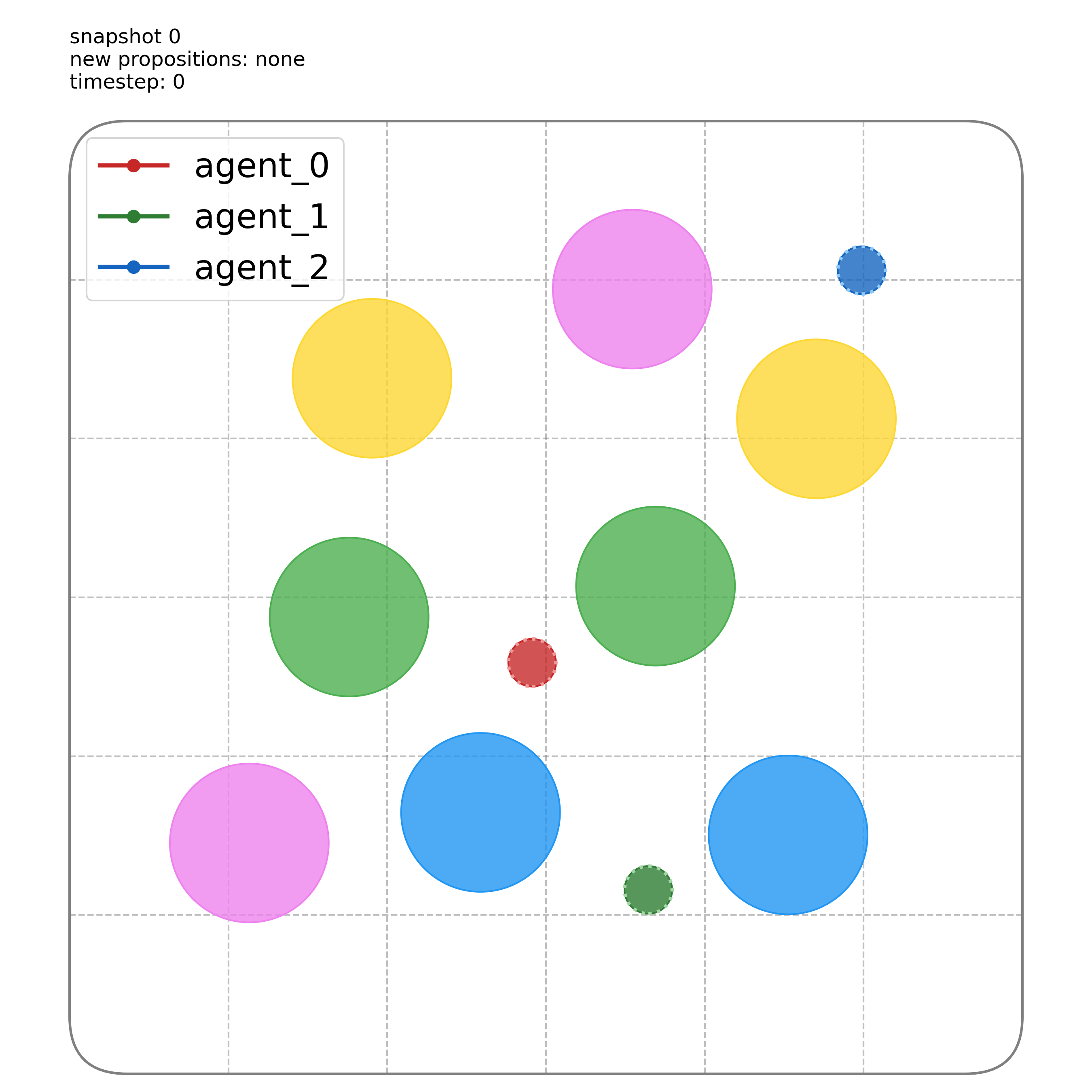}
        \caption{$t=0$\protect\\$\{\}$}
    \end{subfigure}
    \hspace{0.01\textwidth}
    \begin{subfigure}{0.23\textwidth}
        \centering
        \includegraphics[width=\textwidth, trim={0 0 0 60}, clip]
        {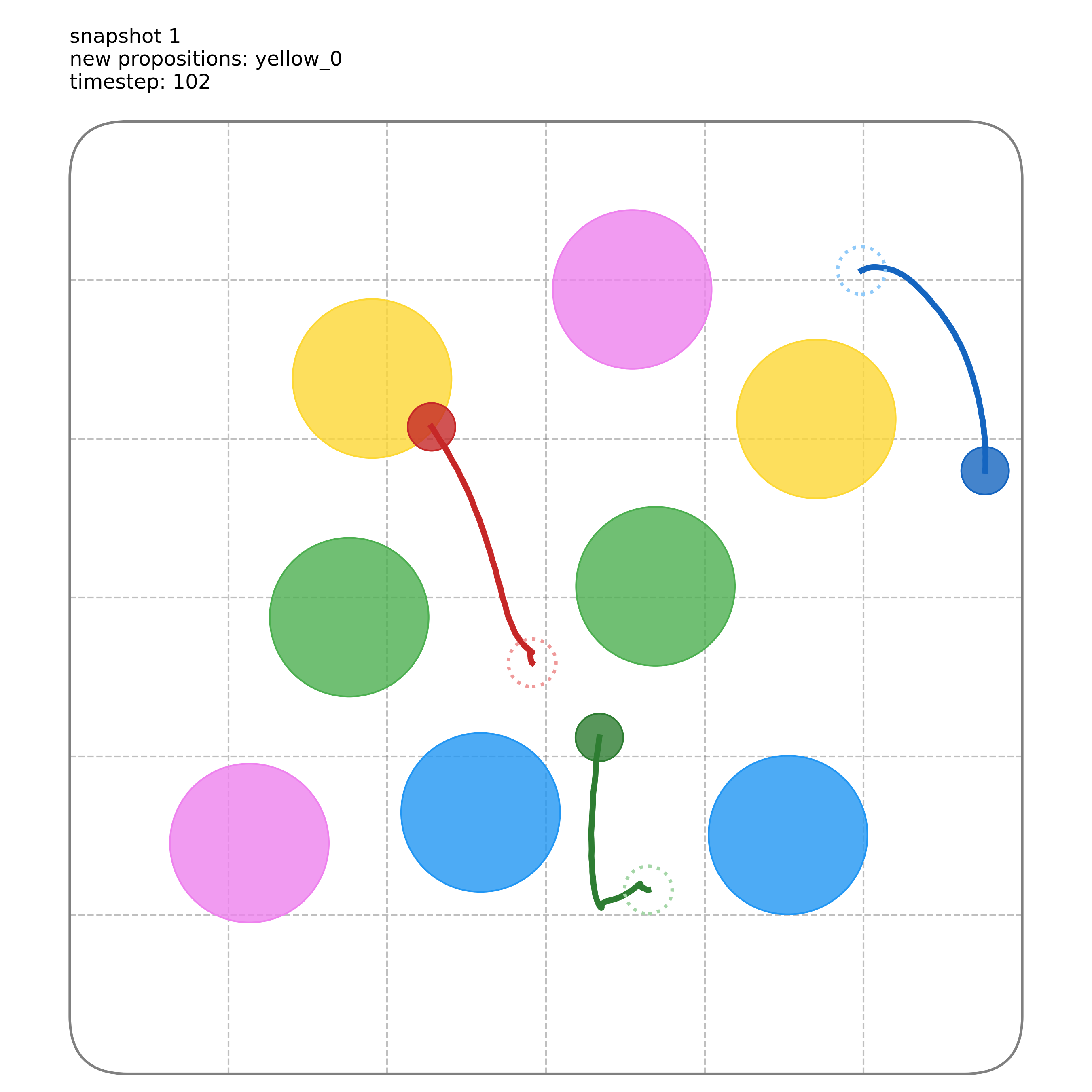}
        \caption{$t=102$\protect\\$\{\mathsf{yellow}_0\}$}
    \end{subfigure}
    \hspace{0.01\textwidth}
    \begin{subfigure}{0.23\textwidth}
        \centering
        \includegraphics[width=\textwidth, trim={0 0 0 60}, clip]
        {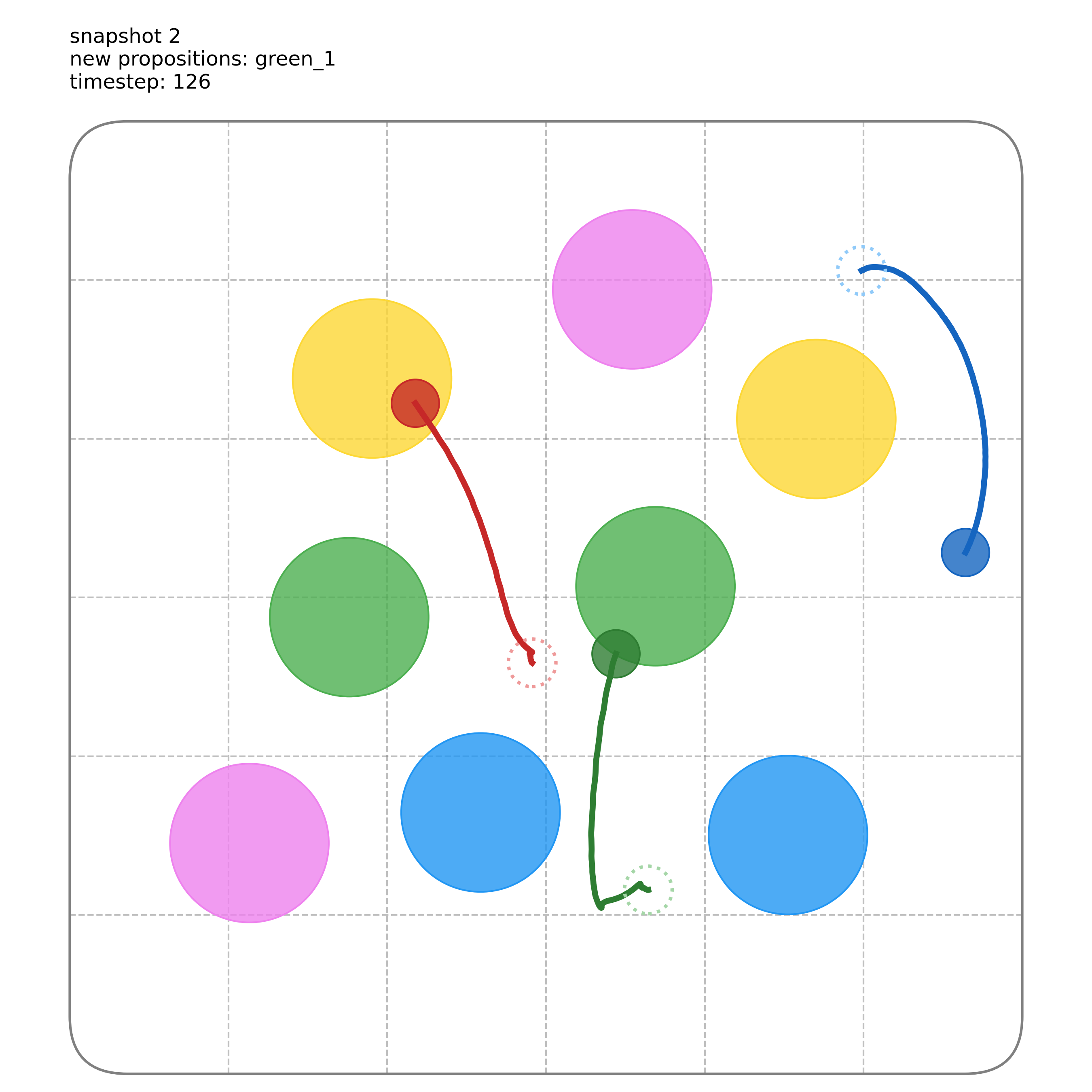}
        \caption{$t=126$\protect\\$\{\mathsf{yellow}_0,\mathsf{green}_1\}$}
    \end{subfigure}
    \hspace{0.01\textwidth}
    \begin{subfigure}{0.23\textwidth}
        \centering
        \includegraphics[width=\textwidth, trim={0 0 0 60}, clip]
        {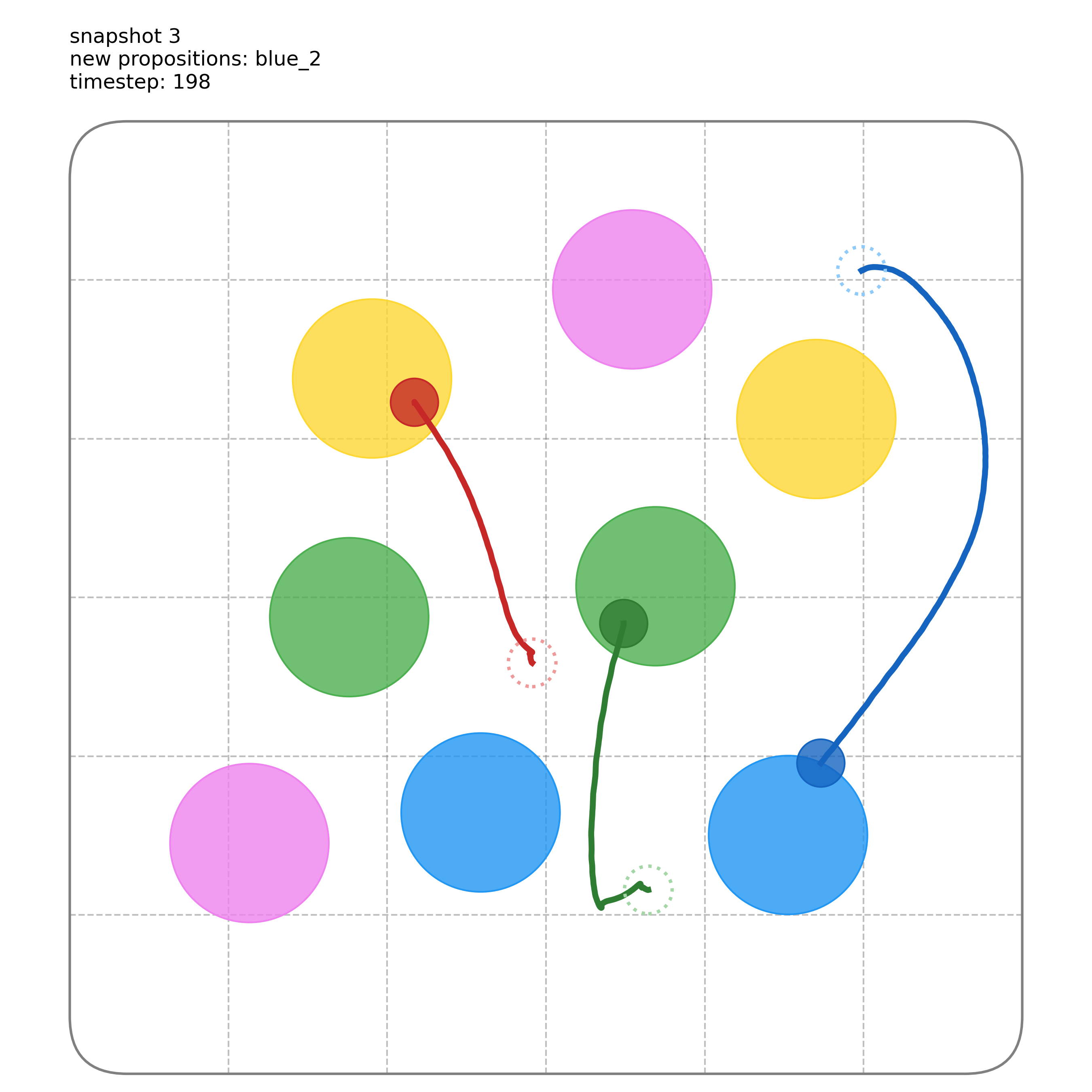}
        \caption{$t=198$\protect\\$\{\mathsf{yellow}_0,\mathsf{green}_1,\mathsf{blue}_2\}$}
    \end{subfigure}

    \vspace{0.2cm}

    \begin{subfigure}{0.23\textwidth}
        \centering
        \includegraphics[width=\textwidth, trim={0 0 0 60}, clip]
        {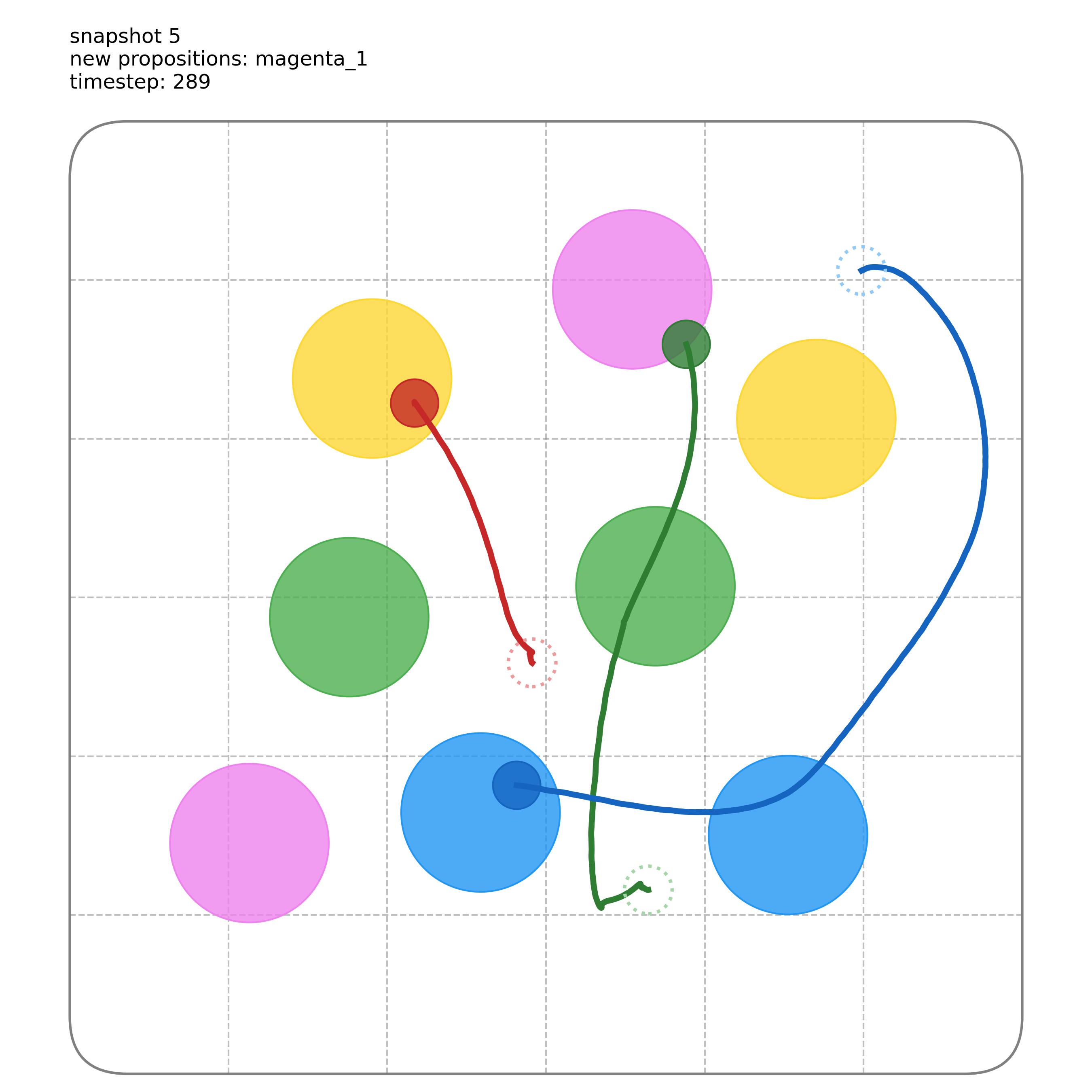}
        \caption{$t=289$\protect\\$\{\mathsf{yellow}_0,\mathsf{magenta}_1,\mathsf{blue}_2\}$}
    \end{subfigure}
    \hspace{0.01\textwidth}
    \begin{subfigure}{0.23\textwidth}
        \centering
        \includegraphics[width=\textwidth, trim={0 0 0 60}, clip]
        {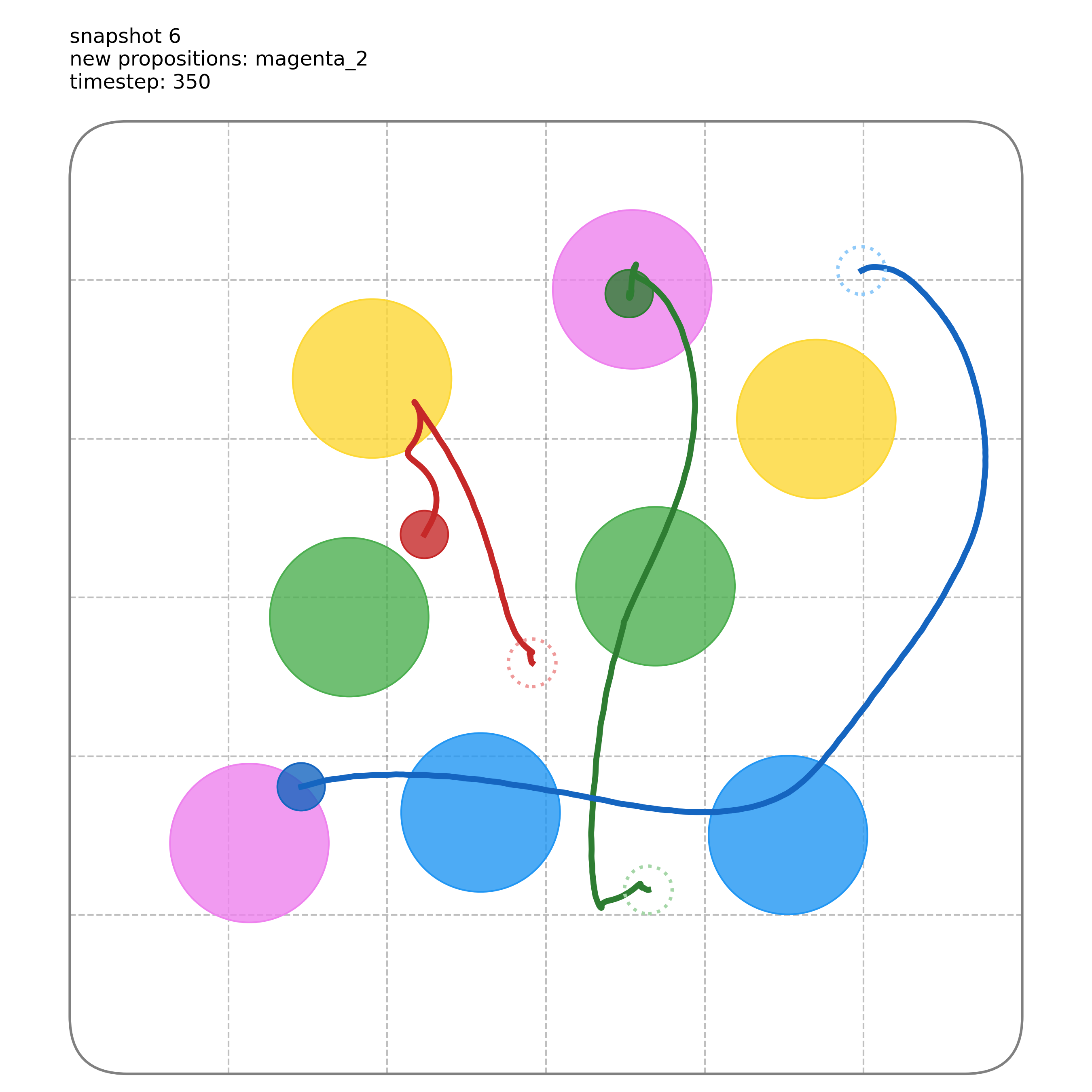}
        \caption{$t=349$\protect\\$\{\mathsf{magenta}_1,\mathsf{magenta}_2\}$}
    \end{subfigure}
    \hspace{0.01\textwidth}
    \begin{subfigure}{0.23\textwidth}
        \centering
        \includegraphics[width=\textwidth, trim={0 0 0 60}, clip]
        {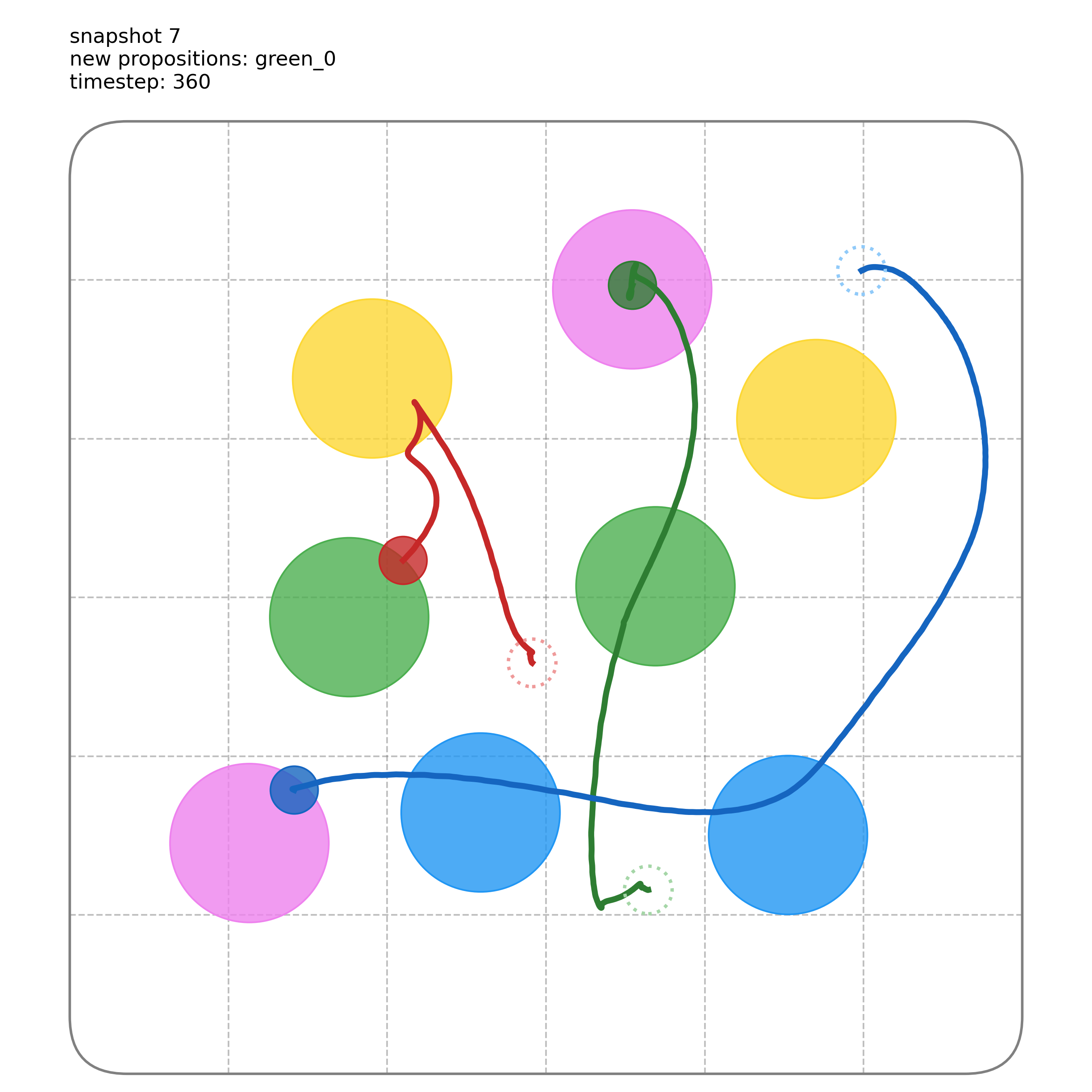}
        \caption{$t=358$\protect\\$\{\mathsf{green}_0,\mathsf{magenta}_1,\mathsf{magenta}_2\}$}
    \end{subfigure}
    \hspace{0.01\textwidth}
    \begin{subfigure}{0.23\textwidth}
        \centering
        \includegraphics[width=\textwidth, trim={0 0 0 60}, clip]
        {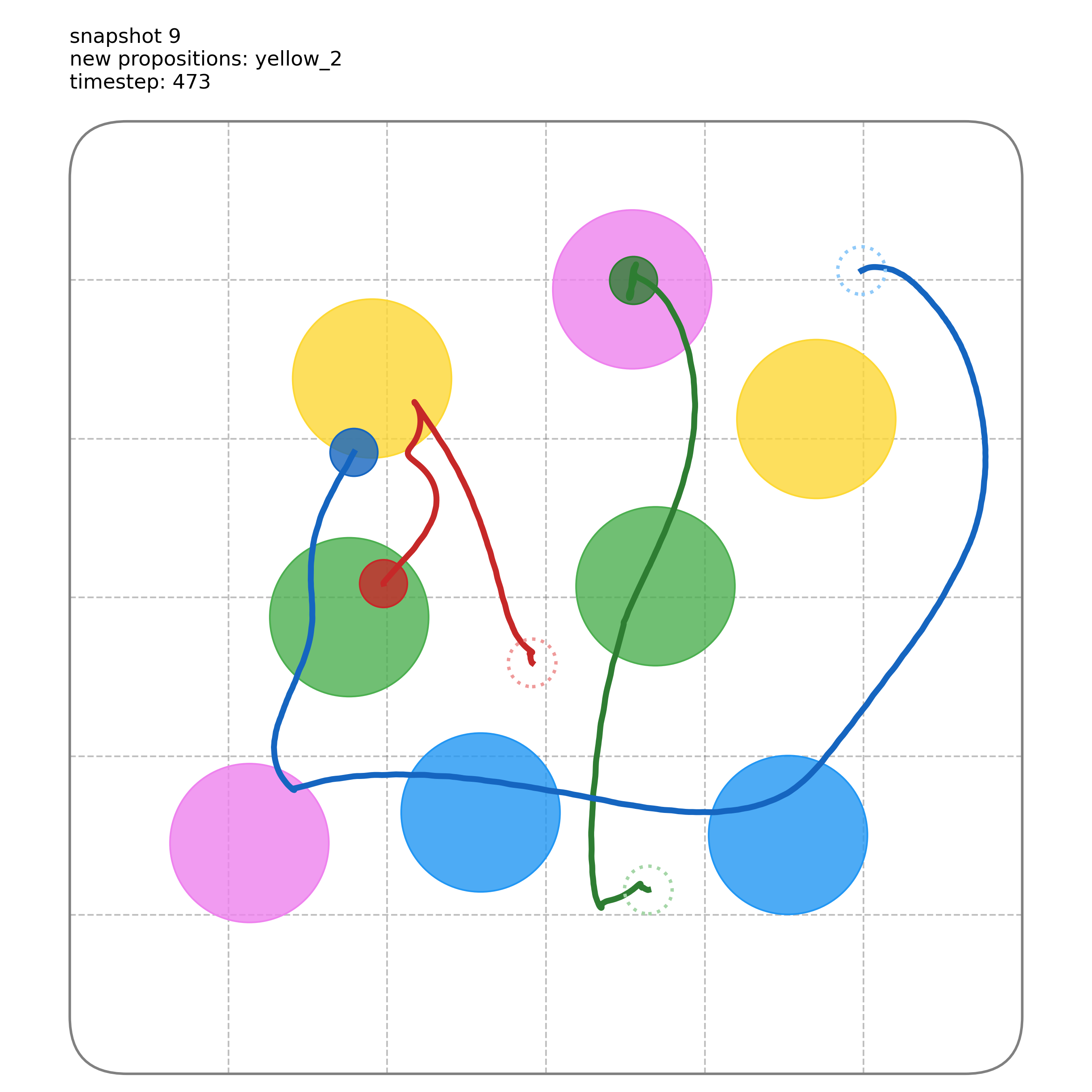}
        \caption{$t=471$\protect\\$\{\mathsf{green}_0,\mathsf{magenta}_1,\mathsf{yellow}_2\}$}
    \end{subfigure}

    \caption{Visualization of the trajectory of specification $\{\varphi_4,\varphi_5,\varphi_6,\varphi_{11},\varphi_{14}\}$. For each subplot, $t$ denotes the timestep, and the set denotes the propositions that are true at the current timestep.}
    \label{fig:qualitative-results}
\end{figure}

\begin{figure}[t]
    \centering
    \captionsetup[subfigure]{
        font=scriptsize,
        justification=centering,
        singlelinecheck=true,
        skip=1pt
    }

    \begin{subfigure}{0.96\textwidth}
        \centering
        \includegraphics[width=\linewidth]
        {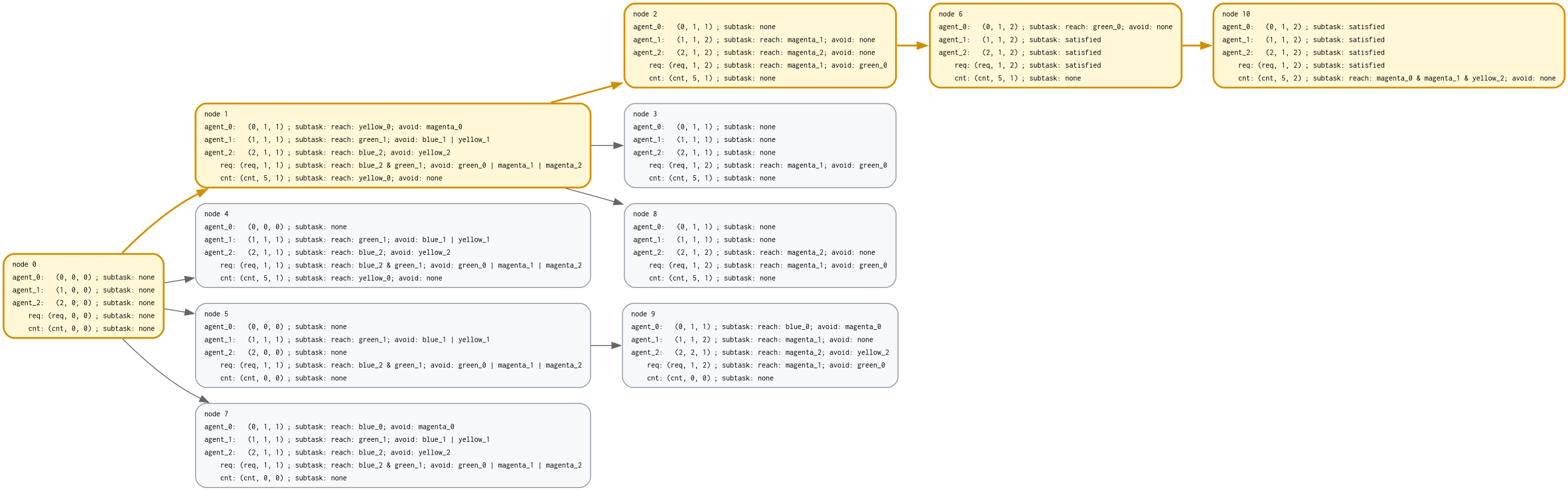}
        \caption{$t=0$}
    \end{subfigure}

    \vspace{2.5mm}

    \begin{subfigure}{0.96\textwidth}
        \centering
        \includegraphics[width=\linewidth]
        {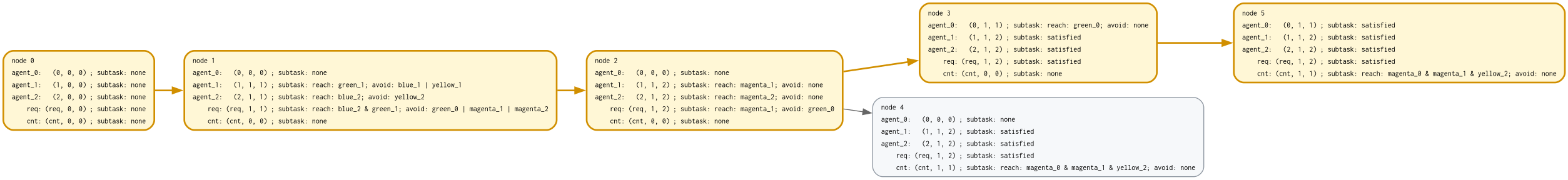}
        \caption{$t=102$}
    \end{subfigure}

    \vspace{2.5mm}

    \begin{subfigure}{0.96\textwidth}
        \centering
        \includegraphics[width=\linewidth]
        {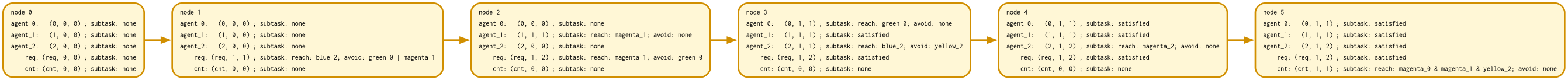}
        \caption{$t=126$}
    \end{subfigure}

    \vspace{2.5mm}

    \begin{subfigure}{0.96\textwidth}
        \centering
        \includegraphics[width=\linewidth]
        {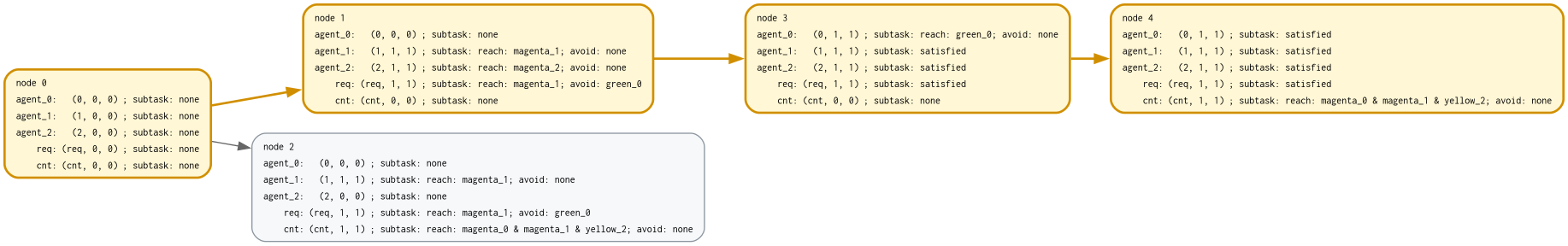}
        \caption{$t=198$}
    \end{subfigure}

    \vspace{2.5mm}

    \begin{subfigure}{0.80\textwidth}
        \centering
        \includegraphics[width=\linewidth]
        {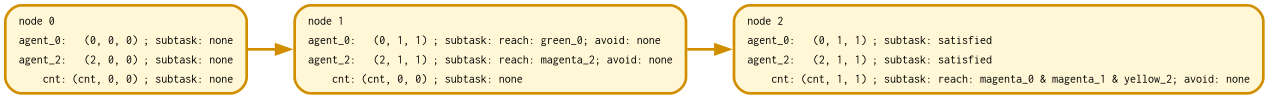}
        \caption{$t=289$}
    \end{subfigure}

    \vspace{3mm}

    \begin{subfigure}[t]{0.47\textwidth}
        \centering
        \includegraphics[width=\linewidth]
        {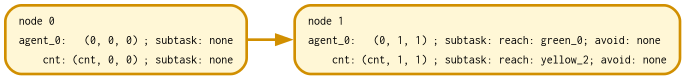}
        \caption{$t=349$}
    \end{subfigure}
    \hfill
    \begin{subfigure}[t]{0.47\textwidth}
        \centering
        \includegraphics[width=\linewidth]
        {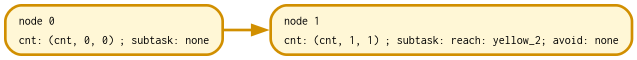}
        \caption{$t=358$}
    \end{subfigure}

    \caption{
    Visualization of the planning process of the high-level sampling-based planner for specification $\{\varphi_4,\varphi_5,\varphi_6,\varphi_{11},\varphi_{14}\}$.
    Each subplot shows the planning state at timestep $t$.
    Yellow blocks denote the selected subtask sequence, while gray blocks denote sampled nodes that are not part of the selected sequence.
    }
    \label{fig:rrt_planner}
\end{figure}

At the beginning of execution, the planner assigns $\sigma_0=(\{\mathsf{yellow}_0\},\{\mathsf{green}_0,\mathsf{magenta}_0\})$ (i.e., reach $\mathsf{yellow}_0$ and avoid $\mathsf{green}_0$ or $\mathsf{magenta}_0$), $\sigma_1=(\{\mathsf{green}_1\},\{\mathsf{blue}_1,\mathsf{magenta}_1,\mathsf{yellow}_1\})$, and $\sigma_2=(\{\mathsf{blue}_2\},\{\mathsf{magenta}_2,\mathsf{yellow}_2\})$.
After $\texttt{agent}_0$ reaches $\mathsf{yellow}$ at $t=102$, its individual task permits progression toward $\mathsf{green}$, but $\mathsf{green}_0$ remains avoided until the prerequisite $\mathsf{magenta}_1$ is satisfied; meanwhile, $\texttt{agent}_1$ and $\texttt{agent}_2$ continue their current subtasks.
At $t=126$, $\texttt{agent}_1$ reaches $\mathsf{green}$, after which $\mathsf{magenta}_1$ remains avoided until $\texttt{agent}_2$ reaches $\mathsf{blue}$, as required by the corresponding $\texttt{dep}$ constraint.
Once $\texttt{agent}_2$ reaches $\mathsf{blue}$ at $t=198$, the planner updates the subtasks to $\sigma_1=(\{\mathsf{magenta}_1\},\{\emptyset\})$ and $\sigma_2=(\{\mathsf{magenta}_2\},\{\emptyset\})$, while $\texttt{agent}_0$ continues to avoid $\mathsf{green}$.
After $\texttt{agent}_1$ first reaches $\mathsf{magenta}$ at $t=289$, the prerequisite for $\mathsf{green}_0$ is satisfied, and the planner assigns $\sigma_0=(\{\mathsf{green}_0\},\{\emptyset\})$ while retaining $\mathsf{magenta}_1$ and $\mathsf{magenta}_2$ as targets to satisfy the synchronous count requirement.
At $t=349$, $\texttt{agent}_1$ and $\texttt{agent}_2$ reach $\mathsf{magenta}$ simultaneously, after which the remaining subtasks are $\sigma_0=(\{\mathsf{green}_0\},\{\emptyset\})$ and $\sigma_2=(\{\mathsf{yellow}_2\},\{\emptyset\})$.
Finally, $\texttt{agent}_0$ reaches $\mathsf{green}$ at $t=358$ and $\texttt{agent}_2$ reaches $\mathsf{yellow}$ at $t=471$, and the specifications are satisfied.

\end{document}